\documentclass[twoside,11pt]{article}
\usepackage{jair, theapa, rawfonts}

\usepackage{algorithm}
\usepackage{algorithmic}

\usepackage{enumitem}
\usepackage{graphicx}
\usepackage{amsthm}
\usepackage{amsmath}
\usepackage{amssymb}
\usepackage{xcolor}
\usepackage{bm}
\usepackage{subfigure}
\newtheorem{assumption}{Assumption}
\newtheorem{lemma}{Lemma}
\newtheorem{theorem}{Theorem}
\newtheorem{remark}{Remark}

\newtheorem{example}{Example}
\begin{document}
	\title{Joint Optimization of Concave Scalarized Multi-Objective Reinforcement Learning with Policy Gradient Based Algorithm}

	\author{\name Qinbo Bai \email bai113@purdue.edu \\
		\addr Purdue University
		\AND
		\name Mridul Agarwal \email agarw180@purdue.edu \\
		\addr Purdue University, 
		\AND
		\name Vaneet Aggarwal \email vaneet@purdue.edu \\
		\addr Purdue University
	}

	\maketitle

	\begin{abstract}
Many engineering problems have multiple objectives, and the overall aim is to optimize a non-linear  function of these objectives. In this paper, we formulate the problem of maximizing a non-linear concave function of multiple long-term objectives. A policy-gradient based model-free algorithm is proposed for the problem. To compute an estimate of the gradient, an asymptotically biased estimator is proposed. The proposed algorithm is shown to achieve convergence to within an $\epsilon$ of the global optima after sampling $\mathcal{O}(\frac{M^4\sigma^2}{(1-\gamma)^8\epsilon^4})$ trajectories where $\gamma$ is the discount factor and $M$ is the number of the agents, thus achieving the same dependence on $\epsilon$ as the policy gradient algorithm for the standard reinforcement learning.\footnote{This is an updated version of the previous manuscript after addressing comments from JAIR reviewers.}

\end{abstract}
	\section{Introduction}
The standard formulation of reinforcement learning (RL), which aims to find the optimal policy to optimize the cumulative reward, has been well studied in the recent years. Compared with the model-based algorithms, model-free algorithms do not require the estimation of the transition dynamics and  can be extended to the continuous space. Value function based  algorithms such as Q-learning \cite{watkins1992q,jin2018}, SARSA \cite{rummery1994line}, Temporal Difference (TD) \cite{sutton1988learning} and policy based algorithms such as policy gradient \cite{sutton2000} and natural policy gradient \cite{Kakade2002} have been proposed based on the Bellman Equation, which is a result of the additive structure for the standard RL.

However, many applications require more general non-linear reward functions. As an example, risk-sensitive  objectives have been considered in \cite{mihatsch2002risk}. \cite{Hazan2019} studies the problem of maximizing the entropy of state-action distribution. Further, many realistic applications have multiple objectives, e.g., capacity and power usage in the communication system \cite{aggarwal2017joint}, latency and energy consumption in queueing systems \cite{badita2020optimal}, efficiency and safety in robotic systems \cite{nishimura2020l2b}. 


In this paper, we consider a setting that jointly optimizes a general concave function of the cumulative reward from multiple objectives. 
{
\begin{equation}
	\max_{\pi} f(J_1^{\pi},\cdots,J_M^{\pi})
\end{equation} 
where $J_m^{\pi}$ is the value function following policy $\pi$ for $m^{th}$ objective and $f$ is the general concave function. The detailed formulation can be found in Section \ref{sec_formulation}.} With this definition, a non-linear concave function of single objective becomes a special case. Further, fair allocation of resources among multiple users require a non-linear function of the rewards to each user (which correspond to the multiple objectives) \cite{Tian2010}, and is thus a special case of this formulation. {In the following, we provide two examples to better motivate our formulation.

\begin{example}
	(Communication System) In the communication system, there is a wireless scheduler to which $M$ users are connected. Each user can exist in two states, \textit{good} or \textit{bad}. The action is the user to which the scheduler allocates the resource. This system has $2^M$ states with $M$ actions. At time $t$, each user $m$ achieves different rates $r_{m,t}$ based on their states and resource allocation. The joint objective function is proportional fairness or sum-logarithmic utility defined as:
\begin{align}
	f(J_1^{\pi},J_2^{\pi},\cdots,J_M^{\pi}) = \sum\nolimits_{m=1}^M \log\left(J_m^\pi\right)
\end{align}
where $J_m^\pi$ is the value function of user $m$ using policy $\pi$.
\end{example}
\begin{example}
	(Queuing System) There is a server serving $M$ queues with Poisson arrivals with different arrival rates. The system state is $M$ dimensional vector of the length of the $M$ queues. The action at each time is the queue which the server serves. At time $t$, each queue $m$ achieves a reward of $1$ unit if a customer from this queue is served.  The joint objective function is $\alpha$ fairness utility (with $\alpha = 2$) defined as:
	
	\begin{align}
		f(J_1^{\pi},J_2^{\pi},\cdots,J_M^{\pi}) = -\sum\nolimits_{m=1}^M \frac{1}{J_m^\pi}
	\end{align}
where $J_m^\pi$ is the value function of queue $m$ using policy $\pi$.

\end{example}

Note that in both  the examples, the value of the function cannot be calculated using reward at time $t$ (or $f(r_{1,t},  \cdots, r_{M,t})$ cannot be used for these problems), as the users which are not allocated wireless resource or the queues which are not served receive $0$ reward and the function value is $-\infty$ . 
}

Such a setup was first considered in \cite{mridulmb}, where a model-based algorithm was proposed for the problem with provable regret guarantees. However, guarantees for model-free algorithm have not been studied to the best of our knowledge, which we focus on. 

We note that the non-linear objective function looses the additive structure, and thus the Bellman's Equation does not work anymore in this setting \cite{mridulmb,Mengdi2020}. Thus, the value function based algorithm do not directly work in this setup. This paper considers a policy-gradient approach and aim to show the global convergence of such policies. Recently, the authors of \cite{Mengdi2020,Mengdi2021} considered the problem for a single-objective over finite state-action space. However, such a problem is open for continuous state action spaces, and for multiple objectives, which is the focus of this paper. In this paper, we consider a  fundamental policy based algorithm, the vanilla policy gradient, and show the global convergence of this policy based on an efficient estimator of the gradient proposed in this paper. 

We note that in standard reinforcement learning, Policy Gradient Theorem \cite{sutton2000} is used to propose an unbiased gradient estimator such as REINFORCE. However, such an approach can not directly give an unbiased estimator in our setting due to the presence of non-linear function (See Lemma \ref{lem_biased_est}). In this paper, we provide a biased estimator for the policy gradient. This biased estimator is then used to prove the global convergence of the policy gradient algorithm.


Our contribution can be summarized as follows.

\begin{itemize}[leftmargin=*]
	\item We consider a new problem statement in reinforcement learning, which aims to jointly optimize a multi-objective problem with concave utility. Such formulation has rarely been considered before.
	\item Due to the existence of concave utility, it is impossible to give an unbiased estimator. Thus, we propose a general biased gradient estimator, which can be applied to both tabular and continuous state-action spaces prove that the bias of the estimator decays at order $\mathcal{O}(1/\sqrt{n})$, where $n$ is the number of trajectories sampled (See Remark \ref{rem_bias_decay}).
	\item We prove the policy gradient algorithm with the proposed estimator converges to the global optimal with error $\epsilon$ using  $\mathcal{O}(\frac{M^4\sigma^2}{(1-\gamma)^8\epsilon^4})$ samples, where $M$ is the number of objectives, $\sigma^2$ is the variance defined in Assumption \ref{ass_bounded_var} and $\gamma$ is the discount factor. As compared to the number of samples  for standard RL with policy gradient algorithm \cite{Yanli2020}, our result has the same dependence on $\epsilon$. 
	{ \item We also study our algorithm empirically. We observe that the proposed method performs better than a naive implementation of RL algorithms where reward at each time step is the value of the concave function of the individual rewards.}
\end{itemize} 

Further, even for the case when there is a non-linear function of a single objective, the approach and results are novel, and have not been considered in the prior works for continuous state-action spaces. 
	\section{Related Work}
\begin{table*}[htbp]
	{ 
		\renewcommand\arraystretch{2}
		\centering
		\resizebox{\textwidth}{!}{
			\begin{tabular}{|c|c|c|c|c|c|}
				\hline
				& Works & Sample Complexity & Objective-Function & Multi-Objective & State Action Space  \\
				\hline
				Model-Based & \cite{mridulmb} & $\Tilde{O}(M^2/\epsilon^2)$ & Concave Scalarization  & Yes & Finite\\
				\hline
				& \cite{cheung2019regret} & $\Tilde{O}(1/\epsilon^2)$ & Special Concave Scalarization \footnotemark[1]  & Yes & Finite\\
				\hline
				Model-Free & \cite{Mengdi2020} & N/A \footnotemark[2] & Concave Utility\footnotemark[3] & No & Finite\\
				\hline
				& \cite{Mengdi2021} & $\Tilde{O}(1/\epsilon^2)$ & Concave Utility\footnotemark[3] & No  & Finite\\
				\hline
				& {\bf This Work} & $\Tilde{O}(M^4/\epsilon^4)$ & Concave Scalarization   & Yes& Infinite\\
				\hline
				& \cite{Yanli2020} & $\Tilde{O}(1/\epsilon^4)$ & Reinforce & No & Infinite\\
				\hline
			\end{tabular}
		}
		\caption{{ Overview of key related works for the problem in this paper. $M$ is the number of objectives and $\epsilon$ is the gap between optimal objective and the objective function following the policy in the proposed algorithm.}}
		\label{tab:algo_comparisons}
	}
\end{table*}

Table \ref{tab:algo_comparisons} summarizes the key related works. The problem has been studied in the tabular model-based setup \cite{mridulmb,cheung2019regret}. For the model-free approach, this is the first paper on guarantees on concave scalarized multi-objective infinite horizon reinforcement learning with large state-action space. As compared to the linear scalarization, biased estimator complicates the analysis, and the approach of finite state-action spaces do not directly extend to our problem. Detailed comparison to the approaches is also provided in the following. 

\footnotetext[1]{{ \cite{cheung2019regret} defines a specialized concave scalarization function, where $f(\bm{J})=\frac{1}{M}\cdot\bigg[\sum_{m=1}^{M}L_mJ_m-\frac{L_0}{2}\min_{u\in U}\bigg\{\sum_{m=1}^{M}(J_m-u_m)^2\bigg\}\bigg]$, where $L_0,\cdots,L_M$ are parameters and $U\in[0,1]^M$ is a convex compact set. The proposed algorithm and the achieved sample complexity is limited to above function and whether it can be extended to the general concave scalarization function is unknown.}}

\footnotetext[2]{{ \cite{Mengdi2020} proposed the Varational Policy Gradient Algorithm to solve the problem. \cite{Mengdi2020}[Theorem 4.5] stated the algorithm requires $O(\epsilon^{-1})$ iterations to achieve $\epsilon$-optimal policy. However, in each iteration, it needs to solve a min-max problem, which is costly even for estimating a single policy gradient.}}

\footnotetext[3]{ \cite{Mengdi2020,Mengdi2021} considered the concave utility function, where the objective is to maximize $g(\bm{\lambda})$, and $\bm{\lambda}$ is a cumulative discounted state-action occupancy measure. Setting $h_m(\bm{\lambda})=\left<\bm{r}_m,\bm{\lambda}\right>$ and defining $g(\bm{\lambda})=f(h_1(\bm{\lambda}),\cdots,h_M(\bm{\lambda}))=f(\bm{J})$, their problem reduces to our formulation. Equivalently, in our problem setup, if we choose  $M = |\mathcal{S}||\mathcal{A}|$ and define the rewards $r_m : \mathcal{S} \times \mathcal{A} \to \{0,1\}$ such that $r_m(s,a) = 1 \quad \text{for exactly one pair } (s,a) \in \mathcal{S} \times \mathcal{A},$ and $r_m(s',a') = 0 \quad \text{for all } (s',a') \neq (s,a),$ with a one-to-one mapping between $m\to (s,a)$ then this reduces precisely to their problem setup. Hence, the two formulations are equivalent.}

\textbf{Policy Gradient with Cumulative Return:} As the core result for policy based algorithms, Policy Gradient Theorem \cite{sutton2000} provides a method to obtain the gradient ascent direction for standard reinforcement learning with the policy parameterization. However, in general, the objective in the reinforcement learning is non-convex with respective to the parameters \cite{Alekh2020}. Thus, the research on policy gradient algorithm focuses on the first order stationary point guarantees for a long time \cite{Tianbing2017,Xu2019,xu2020}. Recently, there is a line of interest on the global convergence result for reinforcement learning. \cite{Kaiqing2019} utilizes the idea of escaping saddle points in policy gradient and shows the convergence to the second order stationary points. \cite{Alekh2020} provides provable global convergence result for direct parameterization and softmax parameterization in the tabular case. For the restrictive parameterization, they propose a variant of NPG, Q-NPG and analyze the global convergence result with the function approximation error for both NPG and Q-NPG. \cite{Mei2020} improves the convergence rate for policy gradient with softmax parameterization from $\mathcal{O}(1/\sqrt{t})$ to $\mathcal{O}(1/t)$ and shows a significantly faster linear convergence rate $\mathcal{O}(\exp(-t))$ for the entropy regularized policy gradient. With actor-critic method \cite{konda2000actor}, \cite{Lingxiao2019} establishes the global optimal result for neural policy gradient method. \cite{bhandari2020global} identifies the structure properties which shows that there are no sub-optimal stationary points for reinforcement learning. \cite{Yanli2020} proposes a general framework of the analysis for policy gradient type of algorithms and gives the sample complexity for PG, NPG and the variance reduced version of them. However, all of the above research have been done on the standard reinforcement learning, where the objective function is the direct summation of the reward. This paper focuses on a joint optimization of multi-objective problem, where multiple objectives are combined with a concave function. 

\textbf{\bf Policy Gradient with General Objective Function:}  Even though standard reinforcement learning has been widely studied, there are few results on the policy gradient algorithm with a general objective function. Some special examples are variance-penalty \cite{YingHuang1994} and maximizing entropy \cite{Hazan2019}. Very recently, \cite{Mengdi2020,Mengdi2021} study the global convergence result of the policy gradient with general utilities. They consider the setting that the objective is a concave function of the state-action occupancy measure, which is similar to our setting. By the method of convex conjugate, \cite{Mengdi2020} proposed a variational policy gradient theorem to obtain the gradient ascent direction and gives the global convergences result of PG with general utilities. Despite enjoying a rate of $\mathcal{O}(1/t)$ in terms of iterations, their algorithm requires an additional saddle point problem to fulfill the gradient update and thus introduce extra computation complexity. \cite{Mengdi2021} further proposes the SIVR-PG algorithm and improves the convergence rate in the same setting. However, the SIVR-PG algorithm requires the estimation of state-action occupancy measure, which means that the algorithm can only be applied to the tabular setting.  We note that  our method does not have such limitation and thus can be applied even if the state and action space is large or continuous. Finally, note that \cite{Mengdi2020,Mengdi2021} improve the previous convergence rate for policy gradient by exploring the hidden convexity of the proposed problem. However, in order to utilize such convexity, they require the assumption that the inverse mapping of visitation measure $\lambda:\Theta\rightarrow \lambda(\Theta)$ exists and the Lipschitz property of such inverse mapping is assumed. It has been shown that such assumption holds for direct parameterization. However, such assumptions for continuous state-action space or other types of parameterization may not be valid.

	\section{Formulation}\label{sec_formulation}
We consider an infinite horizon discounted Markov Decision Process (MDP) $\mathcal{M}$ defined by the tuple $(\mathcal{S},\mathcal{A},\mathbb{P},r_1,r_2,\cdots,r_M,\gamma,\rho)$, where $\mathcal{S}$ and $\mathcal{A}$ denote the  state and  action space, respectively. $\mathbb{P}: \mathcal{S}\times\mathcal{A}\rightarrow\Delta^{\mathcal{S}}$ (where $\Delta^{\mathcal{S}}$ is a probability simplex over $\mathcal{S}$) denotes the transition probability distribution from a state-action pair to another state. $M$ denotes the number of objectives and $r_m: \mathcal{S}\times\mathcal{A}\rightarrow \mathbb{R}$ denotes the reward for the $m^{th}$ objective. $\gamma\in(0,1)$ is the discounted factor and $\rho: \mathcal{S}\rightarrow \Delta^{\mathcal{S}}$ is the distribution for initial state. In this paper, we make following assumption.

\begin{assumption}\label{ass_bound_reward}
	The absolute value of the reward functions $r_m,m\in[M]$ is bounded by some constant. Without loss of generality, we assume $r_m\in[0,1],\forall m\in[M]$.
\end{assumption} 

Define a joint stationary policy $\pi:\mathcal{S}\rightarrow\Delta^{\mathcal{A}}$ that maps a state $s\in\mathcal{S}$ to a probability distribution of actions with a probability assigned to each action $a\in\mathcal{A}$. At the beginning of the MDP, an initial state $s_0\sim\rho$ is given and the agent makes a decision $a_0\sim\pi(\cdot\vert s_0)$. The agent receives $M$ reward $r_m(s_0,a_0)$ and then transits to a new state $s_1\sim\mathbb{P}(\cdot\vert s_0,a_0)$. We define the value function $J_m^{\pi}$ for the $m^{th}$ objective following policy $\pi$ as a discounted sum of reward over infinite horizon.

\begin{equation}\label{eq:expected_V}
	J_m^{\pi}=\mathbf{E}_{\rho,\pi,\mathbb{P}}\bigg[\sum_{t=0}^{\infty}\gamma^tr_m(s_t,a_t)\bigg]
\end{equation} 
where $s_0\sim\rho$, $a_t\sim\pi(\cdot\vert s_t)$ and $s_{t+1}\sim\mathbb{P}(\cdot\vert s_t,a_t)$.
{
	Similarly, we define the state value function $V_m^{\pi}(s)$ and state-action value function $Q_m^{\pi}(s,a)$
	\begin{equation}\label{eq:VandQ}
		\begin{aligned}
		V_m^{\pi}(s)&=\mathbf{E}_{\pi,\mathbb{P}}\bigg[\sum_{t=0}^{\infty}\gamma^tr_m(s_t,a_t)\bigg|s_0=s\bigg]\\
		Q_m^{\pi}(s,a)&=\mathbf{E}_{\pi,\mathbb{P}}\bigg[\sum_{t=0}^{\infty}\gamma^tr_m(s_t,a_t)\bigg|s_0=s,a_0=a\bigg]
		\end{aligned}
	\end{equation}
}	
The agent aims to maximize the joint objective function $f:\mathbb{R}^M\rightarrow\mathbf{R}$, which is a function of the long-term discounted reward of each objective. Formally, the problem is written as

\begin{equation}\label{eq:origin_problem}
	\max_{\pi}f(J_1^{\pi},J_2^{\pi},\cdots,J_M^{\pi})
\end{equation}
We consider a policy-gradient based algorithm on this problem and parameterize the policy $\pi$ as $\pi_\theta$ for some parameter $\theta\in\Theta$ such as softmax parameterization or a deep neural network. Commonly, the log-policy function $\log\pi_\theta(a\vert s)$ is called log-likelihood function and we make the following assumption. 
	\begin{assumption}\label{ass_score}
		The log-likelihood function is $G$-Lipschitz and $B$-smooth. Formally,
		\begin{equation}
			\begin{aligned}
			&\Vert \nabla_\theta\log\pi_\theta(a\vert s)\Vert\leq G\quad\forall \theta\in\Theta,\forall (s,a)\in\mathcal{S}\times\mathcal{A}\\
			&\Vert \nabla_\theta\log\pi_{\theta_1}(a\vert s)-\nabla_\theta\log\pi_{\theta_2}(a\vert s)\Vert\leq B\Vert \theta_1-\theta_2\Vert\quad\forall \theta_1,\theta_2 \in\Theta,\forall (s,a)\in\mathcal{S}\times\mathcal{A}
			\end{aligned}
		\end{equation}
	{ We consider all norms in this paper, unless explicitly mentioned, as L2-norm.}
	\end{assumption}
	\begin{remark}
		The Lipschitz and smoothness properties for the log-likelihood are quite common in the field of policy gradient algorithm \cite{Alekh2020,Mengdi2021,Yanli2020}. Such properties can also be verified for simple parameterization such as Gaussian policy.   
	\end{remark}
	Define the value function vector $\bm{J}^{\pi_\theta}=(J_1^{\pi_\theta},\cdots,J_M^{\pi_\theta})$. The original problem, Eq. \eqref{eq:origin_problem}, can be rewritten as
	\begin{equation}\label{eq:parameterized_problem}
		\max_{\theta\in\Theta}f(\bm{J}^{\pi_\theta})
	\end{equation}
	We make the following assumptions on the objective function $f$:
	\begin{assumption}\label{ass_concave_obj}
		The objective function $f$ is jointly concave. Hence for any arbitrary distribution $\mathcal{D}$, the following holds.
		\begin{equation}
			f(\mathbf{E}_{\bm{x} \sim\mathcal{D}}[\bm{x}])\geq \mathbf{E}_{\bm{x} \sim\mathcal{D}}[f(\bm{x})])\quad\forall \bm{x}\in\mathbb{R}^M
		\end{equation}
	\end{assumption}
	\begin{remark}(Non-Concave Optimization)
		It is worth noticing that the above problem is a non-concave optimization problem despite the above joint-concave assumption on the objective function. This  is because the parameterized value function $\bm{J}_m^{\pi_\theta}$ is non-concave with respect to $\theta$ (See Lemma 3.1 in \cite{Alekh2020}). Thus, the standard theory from convex optimization can't be directly applied to this problem.
	\end{remark}
	\begin{assumption}\label{ass_partial_grad}
		All partial derivatives of function $f$ are assumed to be locally $L_f$-Lipschitz functions. Formally,
		\begin{equation}
			\begin{aligned}
			&\vert \frac{\partial f}{\partial x_i}(\bm{y}_1)-\frac{\partial f}{\partial x_i}(\bm{y}_2)\vert\leq L_f\Vert \bm{y}_1-\bm{y}_2\Vert\\
			&\quad \forall\bm{y}_1,\bm{y}_2\in[0,\frac{1}{1-\gamma}]^M,\forall i\in[M]
			\end{aligned}
		\end{equation}
	\end{assumption}
	\begin{remark}
		By Assumption \ref{ass_bound_reward},  $J_m^{\pi_\theta}$ is bounded in $[0,\frac{1}{1-\gamma}]$. Thus, it is enough to assume the locally Lipschitiz property for the partial derivatives of the objective. Such an assumption has also been adopted widely for the general objective function \cite{Mengdi2020,Mengdi2021}.
	\end{remark}
	\noindent Finally, based on the Assumption. \ref{ass_partial_grad}, we derive the following result for the objective function.
	\begin{lemma}\label{lem_bound_parital}
		All partial derivative functions of $f$ are locally bounded by a constant. Formally,
		\begin{equation}
			\bigg| \frac{\partial f}{\partial x_i}(\bm{y})\bigg|\leq C\quad \forall \bm{y}\in[0,\frac{1}{1-\gamma}]^M,\forall i\in[M]
		\end{equation}
	\end{lemma}
	\begin{proof}
		By Assumption \ref{ass_partial_grad}, the partial derivative function is locally Lipschitz and thus is continuous on the set $[0,\frac{1}{1-\gamma}]^M$, which is compact. Since a continuous function with a compact set is bounded, the result follows.
	\end{proof}

Further discussions on the assumptions are provided in Appendix \ref{app_discussion}. 
	\section{Policy Gradient Method for Joint Optimization}\label{sec_estimator}
Policy gradient algorithm aims to update the parameter with the iteration
\begin{equation}
	\theta^{k+1}=\theta^k+\eta\nabla_\theta f(\bm{J}^{\pi_{\theta^k}})
\end{equation}
where $\eta$ is the step size. However, it is impossible to compute the true gradient because the transition dynamics is unknown in practice. Thus, an estimator for the true gradient is necessary. From the Chain Rule, the gradient for the objective function is (the detailed computation is in appendix \ref{app_grad_computation})
\begin{equation}\label{eq:gradient_origin}
	\nabla_\theta f(\bm{J}^{\pi_\theta})=\mathbf{E}_{\tau\sim p(\tau\vert \theta)}\bigg[\bigg(\sum_{t=0}^\infty \nabla_\theta \log \pi_\theta(a_t\vert s_t)\bigg)\bigg(\sum_{m=1}^{M}\frac{\partial f}{\partial J_m^\pi}\big(\sum_{t=0}^\infty \gamma^tr_m(s_t,a_t)\big)\bigg)\bigg]
\end{equation}
In this section, we firstly propose a biased estimator and bound the bias. The policy-gradient algorithm is also formally described based on the estimator. Finally, we analyze some properties of the objective function, which will be used in the proof of the main result.
	
\subsection{Proposed Estimator}
The REINFORCE estimator of Eq. \eqref{eq:gradient_origin} for $\nabla_\theta f(\bm{J}^{\pi_\theta})$ can be considered as a sampled version of it, and it can be directly derived as
\begin{equation}\label{eq:est_origin}
	g(\tau_i,\tau_{j=1:N_2}\vert \theta)=\sum_{t=0}^\infty \nabla_\theta \log\pi_\theta(a_t^i\vert s_t^i)\bigg(\sum_{m=1}^{M}\big(\frac{\partial f}{\partial J_m^\pi}\bigg|_{J_m^\pi=\hat{J}_m^{\pi}}\big)\cdot\big(\sum_{h=0}^\infty \gamma^hr_m(s_h^i,a_h^i)\big)\bigg)
\end{equation}
where
\begin{equation}
	\hat{J}_{m}^{\pi}=\frac{1}{N_2}\sum_{j=1}^{N_2}\sum_{t=0}^{\infty}\gamma^tr_m(s_t^j,a_t^j)
\end{equation}
and $N_2$ is the number of trajectories of $\tau_j$ that we need to sample to estimate $\frac{\partial f}{\partial J_m^\pi}$. Notice that the trajectories $\tau_i=(s_0^i,a_0^i,s_1^i,a_1^i,\cdots)$ and $\tau_j=(s_0^j,a_0^j,s_1^j,a_1^j,\cdots)$ are sampled independently from the distribution $p(\tau\vert \theta)$. However, notice that { in general} the proposed estimator is not unbiased due to the concavity of the function $f$ (See Lemma \ref{lem_biased_est} in Appendix \ref{sec_app_bias} for detail). Moreover, the estimator in Eq. \eqref{eq:est_origin} is unachievable because it requires a sum over infinite range of $t$. Thus, we define a truncated version of Eq. \eqref{eq:est_origin} as
	
\begin{equation}\label{eq:est_truncated}
	g(\tau_i^H,\tau_{j=1:N_2}^H\vert \theta)=\sum_{t=0}^{H-1} \nabla_\theta \log\pi_\theta(a_t^i\vert s_t^i)\bigg(\sum_{m=1}^{M}\big(\frac{\partial f}{\partial J_m^\pi}\bigg|_{J_m^\pi=\hat{J}_{m,H}^{\pi}}\big)\cdot\big(\sum_{h=0}^{H-1} \gamma^hr_m(s_h^i,a_h^i)\big)\bigg)
\end{equation}
where
\begin{equation}\label{eq:def_Jm}
	\hat{J}_{m,H}^{\pi}=\frac{1}{N_2}\sum_{j=1}^{N_2}\sum_{t=0}^{H-1}\gamma^tr_m(s_t^j,a_t^j)
\end{equation}
Notice that removing the past reward from the return doesn't change the expectation value \cite{Peters2008}. Thus, we can rewrite Eq. \eqref{eq:est_truncated} as a PGT estimator.
\begin{equation}\label{eq:est_pgt_truncated}
		g(\tau_i^H,\tau_{j=1:N_2}^H\vert \theta)=\sum_{t=0}^{H-1} \nabla_\theta \log\pi_\theta(a_t^i\vert s_t^i)\bigg(\sum_{m=1}^{M}\big(\frac{\partial f}{\partial J_m^\pi}\bigg|_{J_m^\pi=\hat{J}_{m,H}^{\pi}}\big)\cdot\big(\sum_{h=t}^{H-1} \gamma^hr_m(s_h^i,a_h^i)\big)\bigg)
\end{equation}
We provide a lemma of equivalence for completeness and the proof is in Appendix \ref{sec_app_equivlence}.
\begin{lemma}\label{lem_grad_simplified}
	The expectation of PGT \eqref{eq:est_pgt_truncated} and REINFORCE \eqref{eq:est_truncated} are the same.
\end{lemma}
In the remaining part of this paper, we denote $g(\tau_i^H,\tau_{j=1:N_2}^H\vert \theta)$ as $g(\tau_i^H,\tau_j^H\vert \theta)$ for simplicity. With this truncated estimator, the proposed algorithm is in Algorithm \ref{alg:PG}. In each iteration of policy gradient ascent, $N_2$ trajectories are sampled in line 3 and used to estimate the value function for each agent. Line 4 samples another $N_1$ trajectories independent of $N_2$ and uses Eq. \eqref{eq:est_truncated} to calculate the gradient estimator. Line 5 and 6 perform one-step gradient descent using the gradient estimator.

\begin{algorithm*}[tb]
	\caption{Policy Gradient for Joint Optimization of Multi-Objective RL}
	\label{alg:PG}
	\begin{algorithmic}[1]
		\STATE Initialize $\theta^0$ and step size $\eta=\frac{1}{4L_J}$
		\FOR{episode $k=0,...,K-1$} 
		\STATE Sample $N_2$ trajectories $\tau_j$ under policy $\theta^k$ of length $H$ and compute $\hat{J}_{m,H}^{\pi}$ by Eq. \eqref{eq:def_Jm}.
		\STATE Sample $N_1$ trajectories $\tau_i$ under policy $\theta^k$ of length $H$ and for each trajectory compute the gradient estimator $g(\tau_i^H,\tau_j^H\vert \theta^k)$ by Eq. \eqref{eq:est_truncated}
		\STATE {Compute the gradient update direction $\omega^k=\frac{1}{N_1}\sum_{i=1}^{N_1}g(\tau_i^H,\tau_j^H\vert \theta^k)$}
		\STATE Update the parameter $\theta^{k+1}=\theta^k+\eta\omega^k$
		\ENDFOR 
	\end{algorithmic}
\end{algorithm*}

	\subsection{Bounding the Bias of the Truncated Estimator}
	To bound the bias of the proposed truncated estimator, we define three auxiliary functions.
	\begin{equation}
		\tilde{g}(\tau_i,\tau_j\vert \theta)=\sum_{t=0}^{\infty} \nabla_\theta \log\pi_\theta(a_t^i\vert s_t^i)\bigg(\sum_{m=1}^{M}\big(\frac{\partial f}{\partial J_m^\pi}\big)\\
		\cdot\big(\sum_{h=t}^{\infty} \gamma^hr_m(s_h^i,a_h^i)\big)\bigg)
	\end{equation}
	\begin{equation}\label{eq:finite_auxi}
		\tilde{g}(\tau_i^H,\tau_j\vert \theta)=\sum_{t=0}^{H-1} \nabla_\theta \log\pi_\theta(a_t^i\vert s_t^i)\bigg(\sum_{m=1}^{M}\big(\frac{\partial f}{\partial J_m^\pi}\big)\\
		\cdot\big(\sum_{h=t}^{H-1} \gamma^hr_m(s_h^i,a_h^i)\big)\bigg)
	\end{equation}
	\begin{equation}\label{eq:truncated_auxi}
		\tilde{g}(\tau_i^H,\tau_j^H\vert \theta)=\sum_{t=0}^{H-1} \nabla_\theta \log\pi_\theta(a_t^i\vert s_t^i)\bigg(\sum_{m=1}^{M}\big(\frac{\partial f}{\partial J_m^\pi}\bigg|_{J_m^\pi=J_{m,H}^{\pi}}\big)\\
		\cdot\big(\sum_{h=t}^{H-1} \gamma^hr_m(s_h^i,a_h^i)\big)\bigg)
	\end{equation}
	where $J_{m,H}^{\pi}=\mathbf{E}\bigg[\sum_{t=0}^{H-1}\gamma^tr_m(s_t,a_t)\bigg]$.
	
	It should be noticed that Eq. \eqref{eq:finite_auxi} and \eqref{eq:truncated_auxi} are different because the value function used in the partial derivatives are truncated in \eqref{eq:truncated_auxi} but not in \eqref{eq:finite_auxi}.	Moreover, Eq. \eqref{eq:truncated_auxi} and the proposed estimator in Eq. \eqref{eq:est_truncated} are also different because \eqref{eq:est_truncated} uses the empirical value for trajectories $\tau_j^H$ while Eq. \eqref{eq:truncated_auxi} uses the expected value. We note that $\tilde{g}(\tau_i,\tau_j\vert \theta)$ is an unbiased estimator for $\nabla_\theta f(\bm{J}^{\pi_\theta})$. Thus, the bias of the truncated estimator Eq. \eqref{eq:est_truncated} can be decomposed as
	\begin{equation}
		\begin{split}
		&\mathbf{E}[g(\tau_i^H,\tau_j^H\vert \theta)]-\nabla_\theta f(\bm{J}^{\pi_\theta})=\mathbf{E}\underbrace{[g(\tau_i^H,\tau_j^H\vert \theta)-\tilde{g}(\tau_i^H,\tau_j^H\vert \theta)]}_{(I)}\\
		&+\mathbf{E}\underbrace{[\tilde{g}(\tau_i^H,\tau_j^H\vert \theta)-\tilde{g}(\tau_i^H,\tau_j\vert \theta)]}_{(II)}+\mathbf{E}\underbrace{[\tilde{g}(\tau_i^H,\tau_j\vert \theta)-\tilde{g}(\tau_i,\tau_j\vert \theta)]}_{(III)}
		\end{split}
	\end{equation}
	which means the bias includes three parts: (I) denotes the bias coming from the finite samples of trajectories $\tau_j$. (II) and (III) denote the bias due to the truncation of trajectories $\tau_j$ and $\tau_i$, respectively. In the following, we give three lemmas to bound each of them. The detailed proofs are provided in Appendix \ref{sec_app_bound_bias}.
	\begin{lemma}\label{lem_bound_bias1}
		For any $\epsilon'>0$ and $p\in(0,1)$, with probability at least $1-p$, if the number of samples for $\tau_j$ satisfies,
		\begin{equation}
			N_2\geq \frac{M(1-\gamma^H)^2}{2(1-\gamma)^2\epsilon'^2}\log(\frac{2MH}{p})
		\end{equation} 
		then for each trajectory $\tau_i$, the first part of bias for the proposed truncated estimator, Eq. \eqref{eq:est_truncated}, is bounded by
		\begin{equation}
			\Vert g(\tau_i^H,\tau_j^H\vert \theta)-\tilde{g}(\tau_i^H,\tau_j^H\vert \theta)\Vert\leq MGL_f\\
			\frac{1-\gamma^H-H\gamma^H(1-\gamma)}{(1-\gamma)^2}\epsilon'
		\end{equation} 
	\end{lemma}
	\begin{lemma}\label{lem_bound_bias2} 
	For each trajectory $\tau_i$, the second part of bias for the proposed truncated estimator, Eq. \eqref{eq:est_truncated}, is bounded by
		\begin{equation}
			\Vert\tilde{g}(\tau_i^H,\tau_j^H\vert \theta)-\tilde{g}(\tau_i^H,\tau_j\vert\theta)\Vert\leq M^{3/2}GL_f\\
			\frac{1-\gamma^H-H\gamma^H(1-\gamma)}{(1-\gamma)^3}\gamma^H
		\end{equation}
	\end{lemma}
	\begin{lemma}\label{lem_bound_bias3} 
		For each trajectory $\tau_i$, the third part of bias for the proposed truncated estimator, Eq. \eqref{eq:est_truncated}, is bounded by
		\begin{equation}
			\Vert\tilde{g}(\tau_i^H,\tau_j\vert \theta)-\tilde{g}(\tau_i,\tau_j\vert\theta)\Vert\leq MGC\frac{\gamma^H(1+H(1-\gamma))}{(1-\gamma)^2}
		\end{equation}
	\end{lemma}
	\begin{remark}\label{rem_bias_decay}
		Combining the Lemmas \ref{lem_bound_bias1}, \ref{lem_bound_bias2}, and \ref{lem_bound_bias3}, it is found that if the length of sampled trajectories is long enough, the bias of the proposed estimator decays as $\mathcal{O}(\frac{1}{\sqrt{N_2}})$.
	\end{remark}

Further, note that the proposed estimator is asymptotically unbiased with respect to $N_2$ and $H$, as the bias reduces with increasing $N_2$ and $H$.
	
	\subsection{Properties of the Objective Function}
	Similar to the truncated estimator, we define a truncated version for the objective function as follows
	\begin{equation*}
		f(\bm{J}_H^{\pi_\theta})=f(\mathbf{E}[\sum_{t=0}^{H-1}\gamma^tr_1(s_t,a_t)],\cdots,\mathbf{E}[\sum_{t=0}^{H-1}\gamma^tr_M(s_t,a_t)])
	\end{equation*}
	In this subsection, we will give some properties of $f(\bm{J}^{\pi_\theta})$ and $f(\bm{J}_H^{\pi_\theta})$. The detailed proofs are provided in  Appendix \ref{sec_app_property}.
	The following lemma shows the smoothness property for $f(\bm{J}^{\pi_\theta})$ and $f(\bm{J}_H^{\pi_\theta})$.
	\begin{lemma}\label{lem_smooth}
		Both the objective function $f(\bm{J}^{\pi_\theta})$ and the truncated version $f(\bm{J}_H^{\pi_\theta})$ are $L_J$-smooth w.r.t. $\theta$, where
		\begin{equation*}
			L_J=\frac{MCB}{(1-\gamma)^2}
		\end{equation*}
	\end{lemma}
	
	It is reasonable to expect that the truncated objective function and the original one can be arbitrary close when the length of horizon is long enough, and the next lemma bounds the gap between original and truncated objective function.
	\begin{lemma}\label{lem_bound_truncated}
		The difference between the gradient of objective function and that of truncated version is bounded by
		\begin{equation}
			\Vert \nabla_\theta f(\bm{J}^{\pi_\theta})-\nabla_\theta f(\bm{J}_H^{\pi_\theta})\Vert\leq \frac{MG\gamma^H}{(1-\gamma)^2}
			\bigg[ \sqrt{M}L_f\frac{1-\gamma^H-H\gamma^H(1-\gamma)}{1-\gamma}+C[1+H(1-\gamma)]\bigg]
		\end{equation}
	\end{lemma}
	
	In order to introduce the following result, it is helpful to define the state visitation measure
	\begin{equation}
		d_{\rho}^{\pi}:=(1-\gamma)\mathbf{E}_{s_0\sim\rho}\bigg[\sum_{t=0}^{\infty}\gamma^t\mathbf{Pr}^{\pi}(s_t=s\vert s_0)\bigg]
	\end{equation}
	where $\mathbf{Pr}^{\pi}(s_t=s\vert s_0)$ denotes the probability that $s_t=s$ with policy $\pi$ starting from $s_0$. In the theoretical analysis of policy gradient for standard reinforcement learning, one key result is the performance difference lemma. In the multi-objective setting, a similar performance lemma is  derived as follows.
	\begin{lemma}\label{lem_performance_diff}
		The difference in the performance for  any policies $\pi_\theta$ and $\pi_{\theta'}$is bounded as follows
		\begin{equation}
			(1-\gamma)[f(\bm{J}^{\pi_\theta})-f(\bm{J}^{\pi_{\theta'}})]\leq \sum_{m=1}^{M}\frac{\partial f(\bm{J}^{\pi_{\theta'}})}{\partial J_m^{\pi_{\theta'}}}
			\mathbf{E}_{s\sim d_\rho^{\pi_\theta}}\mathbf{E}_{a\sim\pi_\theta(\cdot\vert s)}\big[A_m^{\pi_{\theta'}}(s,a)\big]
		\end{equation}
		{ where $A_m^{\pi}(s,a)=V_m^{\pi}(s)-Q_m^{\pi}(s,a)$ is the advantage function.}
	\end{lemma}
	
	\section{Main Result}\label{sec_result}
	Before stating the convergence result for the policy gradient algorithm, we describe the following assumptions which will be needed for the main result. 
	
	\begin{assumption}\label{ass_bounded_var}
		The auxiliary estimator $\tilde{g}(\tau_i^H,\tau_j^H\vert \theta)$ defined in Eq. \eqref{eq:truncated_auxi} has bounded variance. Formally,
		{
		\begin{equation}
			Var(\tilde{g}(\tau_i^H,\tau_j^H\vert \theta)):=\mathbf{E}[\Vert \tilde{g}(\tau_i^H,\tau_j^H\vert \theta)-\mathbf{E}[\tilde{g}(\tau_i^H,\tau_j^H\vert \theta)]\Vert^2]\leq \sigma^2
		\end{equation}
		}
		for any $\theta$ and $\tau_i^H,\tau_j^H\sim p^H(\cdot\vert\theta)$, where $p^H(\cdot\vert\theta)$ is a truncated version of $p(\cdot\vert\theta)$ defined in Eq. \eqref{eq:traj_dis}.
	\end{assumption}
	\begin{remark}
		In the standard reinforcement learning problem, it is common to assume that variance of the estimator is bounded \cite{Yanli2020}, \cite{Xu2019} and \cite{xu2020}. Such assumption has been verified for Gaussian policy \cite{tingting2011} and \cite{Pirotta2013}. By  Lemma \ref{lem_bound_parital}, it can be verified similarly in the  multi-objective setting.
	\end{remark}
	\begin{assumption}\label{ass_pd}
		For all $\theta\in\mathbb{R}^d$, the Fisher information matrix induced by policy $\pi_{\theta}$ and initial state distribution $\rho$ satisfies
		\begin{equation}
			\begin{aligned}
			F_\rho(\theta)&=\mathbf{E}_{s\sim d_\rho^{\pi_\theta}}\mathbf{E}_{a\sim\pi_\theta}[\nabla_{\theta}\log\pi_\theta(a|s)\nabla_\theta\log\pi_\theta(a|s)^T]\\
			&\succeq\mu_F\cdot \mathbf{I}_d
			\end{aligned}
		\end{equation}
		for some constant $\mu_F>0$
	\end{assumption}
	\begin{remark}
		The positive definiteness assumption is standard in the field of policy gradient based algorithms \cite{Kakade2002,Peters2008,Yanli2020,Kaiqing2019}. A common example which satisfies such assumption is Gaussian policy with mean parameterized linearly (See Appendix B.2 in \cite{Yanli2020}).
	\end{remark}
	\begin{assumption}\label{ass_transfer_error}
		Define the transferred function approximation error as below
		\begin{equation}\label{eq:transfer_error}
			L_{d_\rho^{\pi^*},\pi^*}(\omega_*^\theta,\theta)=\mathbf{E}_{s\sim d_\rho^{\pi^*}}\mathbf{E}_{a\sim\pi^*(\cdot\vert s)}\bigg[\bigg(\nabla_\theta\log\pi_{\theta}(a\vert s)\cdot
			(1-\gamma)\omega_*^\theta-\sum_{m=1}^{M}\frac{\partial f(\bm{J}^{\pi_\theta})}{\partial J_m^{\pi_\theta}}A_m^{\pi_\theta}(s,a)\bigg)^2\bigg]
		\end{equation}
		We assume that this error satisfies $L_{d_\rho^{\pi^*},\pi^*}(\omega_*^\theta,\theta)\leq \epsilon_{bias}$ for any $\theta\in\Theta$, {where $\pi^*$ is the optimal policy} and $\omega_*^{\theta}$ is given as
		\begin{equation}\label{eq:NPG_direction}
			\omega_*^{\theta}=\arg\min_{\omega}\mathbf{E}_{s\sim d_\rho^{\pi_{\theta}}}\mathbf{E}_{a\sim\pi_{\theta}(\cdot\vert s)}\bigg[[\nabla_\theta\log\pi_{\theta}(a\vert s)\cdot(1-\gamma)\omega-\sum_{m=1}^{M}\frac{\partial f(\bm{J}^{\pi_{\theta}})}{\partial J_m^{\pi_{\theta}}}A_m^{\pi_{\theta}}(s,a)]^2\bigg]
		\end{equation}
		It can be shown that $\omega_*^\theta$ is the exact Natural Policy Gradient (NPG) update direction.
	\end{assumption}
	\begin{remark}
		By  Eq. \eqref{eq:transfer_error} and \eqref{eq:NPG_direction}, the transferred function approximation error expresses an approximation error with distribution shifted to $(d_\rho^{\pi^*},\pi^*)$. With the softmax parameterization or linear MDP structure \cite{jin2020provably}, it has been shown that $\epsilon_{bias}=0$ \cite{Alekh2020}. When parameterized by the restricted policy class, $\epsilon_{bias}>0$ due to $\pi_\theta$ not containing all policies. However, for a rich neural network parameterization, the $\epsilon_{bias}$ is small \cite{Lingxiao2019}. Similar assumption has been adopted in \cite{Yanli2020} and \cite{Alekh2020}. 
	\end{remark}
	\begin{remark}
		Due to there existing 7 assumption in the paper, we give a further discussion on all assumptions in Appendix \ref{app_discussion}
	\end{remark}

	\subsection{Global Convergence in Multi-Objective Setting}

	Inspired by the global convergence analysis framework for policy gradient in \cite{Yanli2020}, we present a general framework for convergence analysis of non-linear multi-objective policy gradient in the following.

	\begin{lemma}\label{lem_framework} (Generalization of Proposition 4.5 in \cite{Yanli2020})
		Suppose a general gradient ascent algorithm updates the parameter in the way
		\begin{equation}
			\theta^{k+1}=\theta^k+\eta\omega^k
		\end{equation}
		When Assumptions \ref{ass_score} and \ref{ass_transfer_error} hold, we have
		\begin{equation}\label{eq:general_bound}
			\begin{split}
			&f(\bm{J}^{\pi^*})-\frac{1}{K}\sum_{k=0}^{K-1}f(\bm{J}^{\pi_{\theta^k}})\leq \frac{\sqrt{\epsilon_{bias}}}{1-\gamma}+ \frac{G}{K}\sum_{k=0}^{K-1}\Vert(\omega^k-\omega_*^k)\Vert_2\\
			&+\frac{B\eta}{2K}\sum_{k=0}^{K-1}\Vert\omega^k\Vert^2+\frac{1}{\eta K}\mathbf{E}_{s\sim d_\rho^{\pi^*}}[KL(\pi^*(\cdot\vert s)\Vert\pi_{\theta^0}(\cdot\vert s))]		\end{split}
		\end{equation}
		where $\omega_*^k:=\omega_*^{\theta^k}$ and is defined in Eq. \eqref{eq:NPG_direction}
	\end{lemma}
	
	\begin{proof}
		We generalize the Proposition 4.5 in \cite{Yanli2020} by using the Lemma \ref{lem_performance_diff} and propose the framework of global convergence analysis in the joint optimization for multi-objective setting. Thus, the framework proposed in the Proposition 4.5 in \cite{Yanli2020} can be considered as a special case. The detailed proof is provided in Appendix \ref{sec_app_framwork}.
	\end{proof}
	\noindent Now, we provide the main result of global convergence for the policy gradient algorithm with multi-objective setting (with detailed proof  in Appendix \ref{sec_app_thm}).
	\begin{theorem}\label{thm1}
		For any $\epsilon>0$, in the Policy Gradient Algorithm \ref{alg:PG} with the proposed estimator in Eq. \eqref{eq:est_truncated}, if step-size $\eta=\frac{1}{4L_J}$, the number of iteration $K=\mathcal{O}(\frac{M}{(1-\gamma)^2\epsilon})$, the length of each trajectory $H=\mathcal{O}\bigg(\log\frac{M}{(1-\gamma)\epsilon}\bigg)$, the number of samples $N_1=\mathcal{O}(\frac{\sigma^2}{\epsilon})$ and  $N_2=\mathcal{O}(\frac{M^3}{(1-\gamma)^6\epsilon})$ achieves the following bound
		\begin{equation}
			f(\bm{J}^{\pi^*})-\frac{1}{K}\sum_{k=0}^{K-1}f(\bm{J}^{\pi_{\theta^k}})\leq \frac{\sqrt{\epsilon_{bias}}}{1-\gamma}+\epsilon
		\end{equation}
		In other words, policy gradient algorithm needs $\mathcal{O}\big(\frac{M^4\sigma^2}{(1-\gamma)^8\epsilon^4}\big)$ trajectories.
	\end{theorem}

\section{Evaluations}\label{evaluations}

\subsection{Simulation Environment}

To validate the understanding of our analysis, we perform evaluations using a queuing environment with multiple objectives and a concave utility combining the objectives.  
The environment is a server serving $M$ queues with Poisson arrivals with different arrival rates. The system state is $M$ dimensional vector of the length of the $M$ queues. The action at each time is the queue which the server serves. At time $t$, each queue $m$ achieves a reward of $1$ unit if a customer from this queue is served.  The joint objective function is $\alpha$-fairness defined as:
\begin{align}
	f(\sum_t r_{1,t}, \cdots, \sum_t r_{K,t}) = -\sum\nolimits_{m=1}^M \frac{H}{\sum_{t=1}^H \gamma^{t-1}r_{m,t}},
\end{align}
where $H$ is the length of the episode set to $500$ steps.

For our queuing environment, we consider a server serving customers coming from $M$ queues. Each queue follows Poisson arrivals with different arrival rates given in Table \ref{tab:arrival_rates}. The server has access to the length of the queues. On observing the length of the queue, the server selects a queue to process. If the a customer from a queue is served, the queue gets a reward of $1$ unit.

\begin{table}[ht]
	\centering
	\begin{tabular}{|c|c|c|c|c|c|c|c|c|}
		\hline
		$M$&$\lambda_1$ &$\lambda_2$ & $\lambda_3$& $\lambda_4$&$\lambda_5$ &$\lambda_6$ & $\lambda_7$& $\lambda_8$  \\
		\hline
		$2$&$0.16$ & $0.64$ & $-$ & $-$& $-$ & $-$& $-$ & $-$\\
		\hline
		$4$&$0.08$ & $0.16$ & $0.24$ & $0.32$& $-$ & $-$& $-$ & $-$\\                
		\hline
		$8$&$0.0125$ & $0.0375$ & $0.0625$ & $0.0875$&$0.1125$ & $0.1375$ & $0.1625$ & $0.1875$\\
		\hline
	\end{tabular}
	\caption{Arrival rates of the multiple queues for Queuing system environment}
	\label{tab:arrival_rates}
\end{table}

\subsection{Simulation Setup}
We use softmax parameterization for implementing our policies. Further, we use PyTorch version 1.0.1 to implement the policies and perform gradient ascent. The experiments are run on a machine with Intel i9 processor with $36$ logical cores running at $3.00$ GHz each. The machines are equipped with Nvidia GeForce RTX 2080 GPU. Each of the $10$ independent runs for both environment took about $500$ seconds to finish. For the gradient ascent of objective, we used PyTorch's Adam \cite{kingma2015adam} optimization with learning rate of $0.005$. Finally, we use the value of $\gamma$ to be $0.9999$.

\subsection{Simulation Results}
We study the impact of the number of trajectories used for gradient estimation. We keep the number of trajectories $N_1 = N_2 = N$ and vary $N$ from $ 4, 16, 64$, and $256$. We also vary the number of objectives $M$ as $2, 4$, and $8$. We observe the convergence rates for a softmax policy parameterization. We also compare our algorithm with a policy gradient algorithm which trains the actor using reward function $r_{train}(t)$ at each time $t$ defined as,
\begin{align}
r_{train}(t) = -\sum_{m=1}^M\frac{t}{\sum_{\tau = 1}^t \gamma^{\tau-1}r_{k,\tau}}.
\end{align}
To implement the policy gradient, we use the REINFORCE algorithm  \cite{williams1992simple}.

We plot the behavior of the policy gradient for joint optimization for different values of $N$ in Figure \ref{fig:convergence_plots}. We  run $10$ independent iterations and plot the mean in solid lines and the shaded region is $\pm$ standard deviation. In Figure \ref{fig:convergence_plots}, for all values of $M$, we find that increasing $N$, the number of trajectories used for sampling gradient of the function, leads to faster convergence of the joint reward objective. We note that the objective value are in different scales, and hence we cannot directly compare the objective values for different $M$.

For $M=2$ (Figure \ref{fig:latency_M_2}), we note that the performance of $N=256, 64$, and $N=16$ are almost similar; but as compared to $N=4$, the performance is significantly better. When $M$ is increased to $4$ (in Figure \ref{fig:latency_M_4}), we observe that $N=256$ and $N=64$ are similar and $N=256$ performs only marginally better as compared to $N=64$. However, now $N=16$ does not perform as well as $N=256$ and $N=64$ but the algorithm is still able to converge to the optimal policy with $N=16$. Finally, for $M=8$, we note that $N=256$ again performs better than $N=64$ and $N=16$ with a lesser variance in the performance. However, for $M=8$, the algorithm with $N=16$ is not able to converge to the optimal policy. We infer that for joint optimization of multiple objectives, it is  necessary to increase the number of trajectories as the number of objectives increase.


\begin{figure}[htbp]
	\centering
	\subfigure[$M = 2$]{
		\includegraphics[trim=0.15in 0.05in 0.5in 0.45in, clip, width=0.47\textwidth]{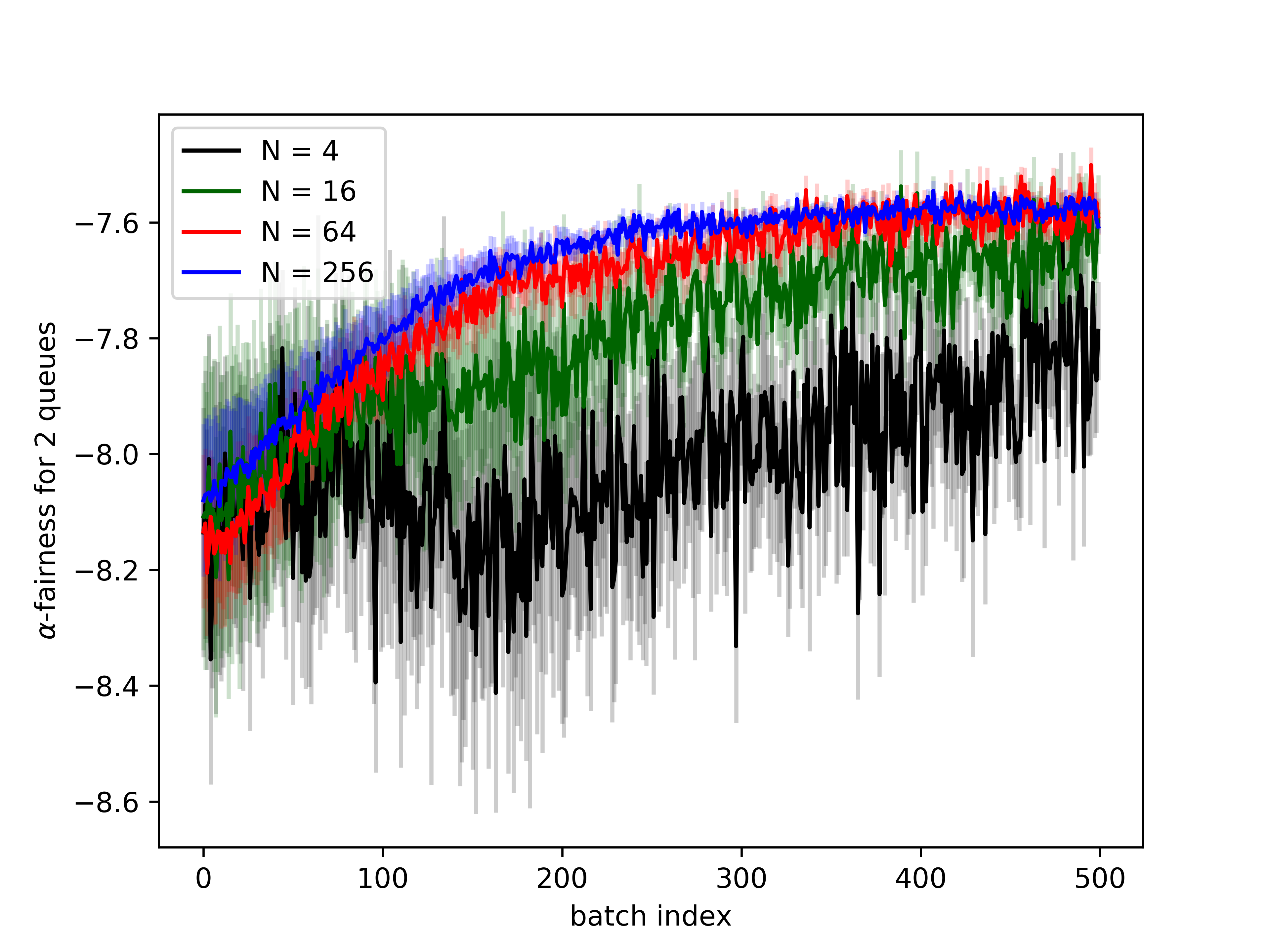}
		\label{fig:latency_M_2}
	}\hspace{-0.1in}
	\subfigure[$M = 4$]{
		\includegraphics[trim=0in 0.05in 0.5in 0.45in, clip, width=0.47\textwidth]{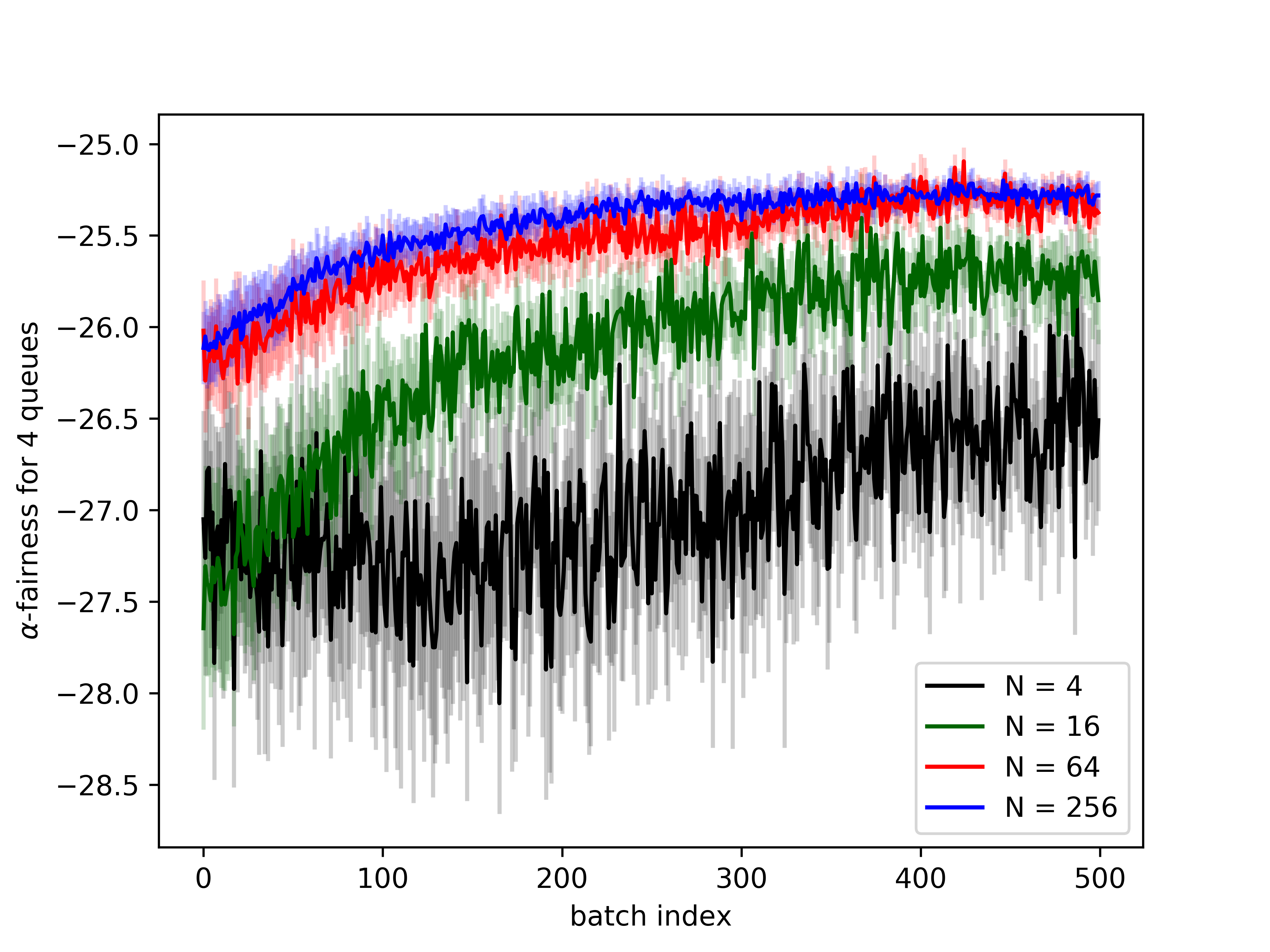}
		\label{fig:latency_M_4}
	}
	\subfigure[$M = 8$]{
		\includegraphics[trim=0in 0.05in 0.5in 0.45in, clip, width=0.47\textwidth]{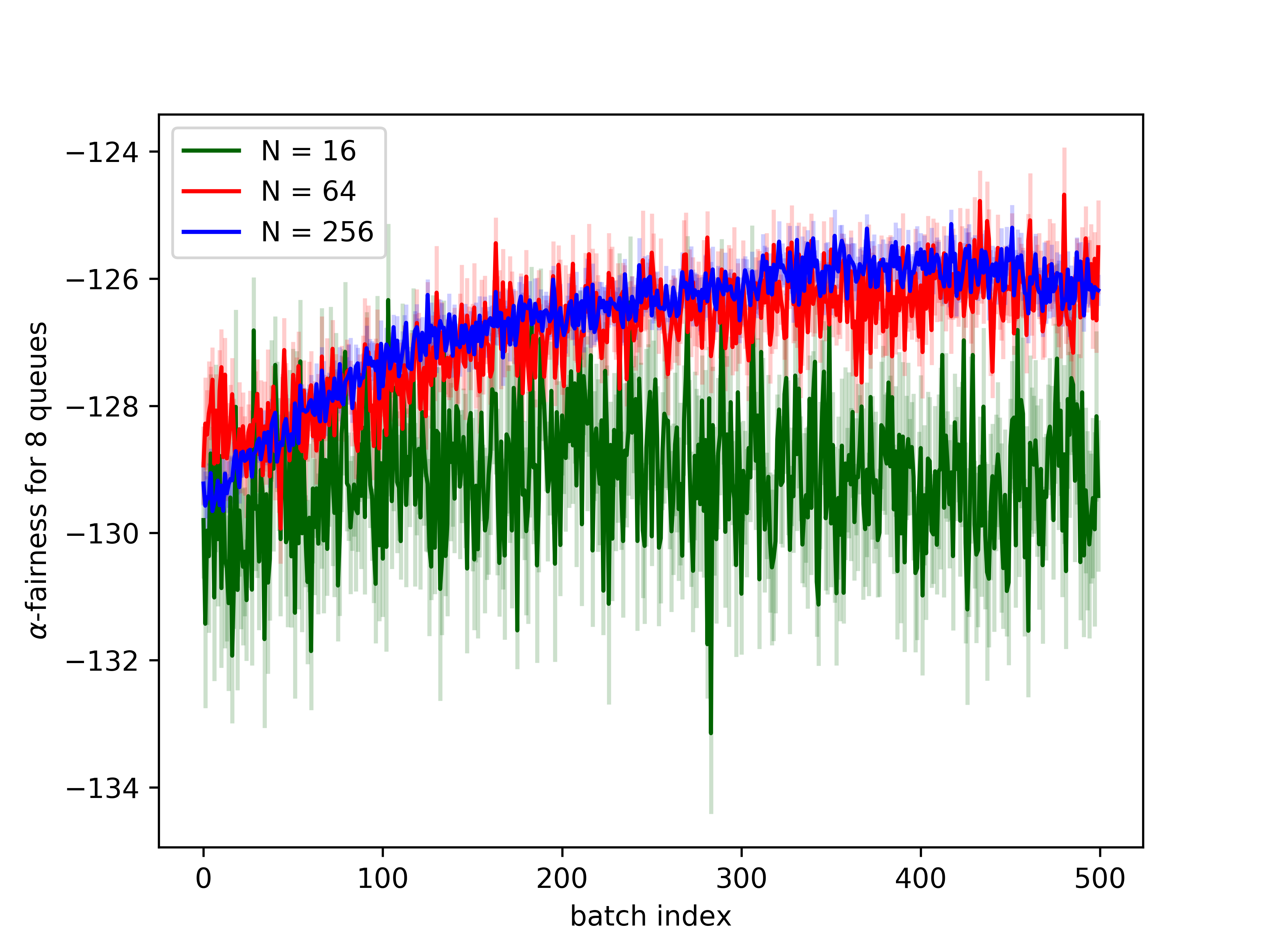}
		\label{fig:latency_M_8}
	}	
	\caption{ Convergence plots for the joint objective policy gradient algorithms for increasing number of queues $M$. As the number of trajectories $N$ used for sampling gradient of the function increase, the convergence becomes steeper. Further, as the number of queues $M$ increase, number of trajectories $N$	is also required to increase to achieve similar performance of levels.}
	\label{fig:convergence_plots}
\end{figure}

We now compare the performance of our proposed algorithm with the REINFORCE algorithm. We present the results in Figure \ref{fig:comparison_with_reinforce}, where we compare the REINFORCE algorithm with varying values for $N$ to compute gradient estimate. We note that the REINFORCE algorithm does not learn a policy which maximizes the objective because the reward at each time step does not provide correct gradient estimate. The performance of the REINFORCE gradient estimate improves with increase in the number of trajectories $N$, but for same number of trajectories, the proposed Algorithm \ref{alg:PG} performs significantly better.  Based on the comparisons, we infer that  using the proposed gradient estimator enables learning optimal policy which maximizes the function $f$. 

\begin{figure}[htbp]
	\centering
	\subfigure[$M = 2$]{
		\includegraphics[trim=0.15in 0.05in 0.5in 0.45in, clip, width=0.47\textwidth]{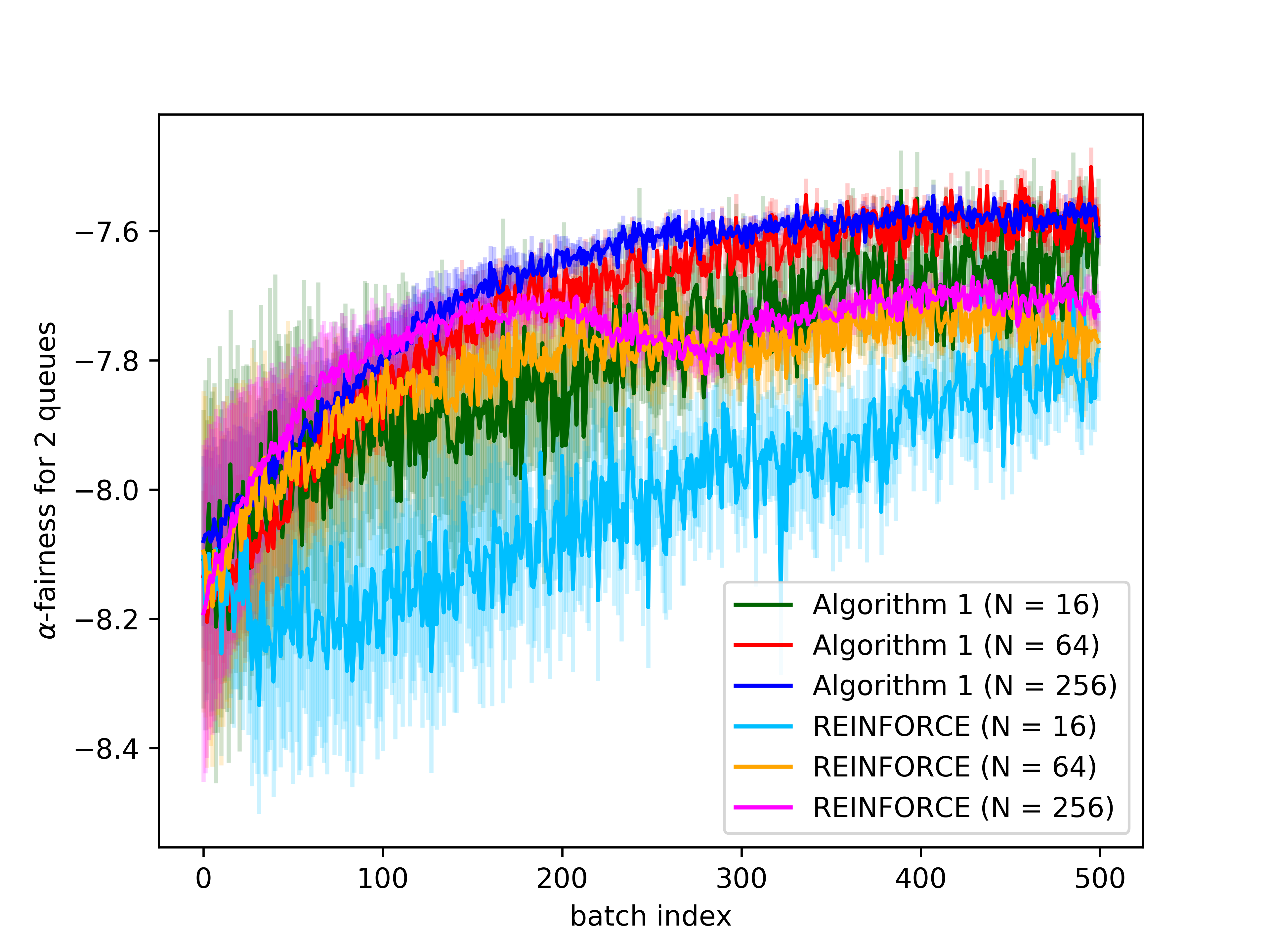}
		\label{fig:reinforce_M_2}
	}\hspace{-0.1in}
	\subfigure[$M = 4$]{
		\includegraphics[trim=0in 0.05in 0.5in 0.45in, clip, width=0.47\textwidth]{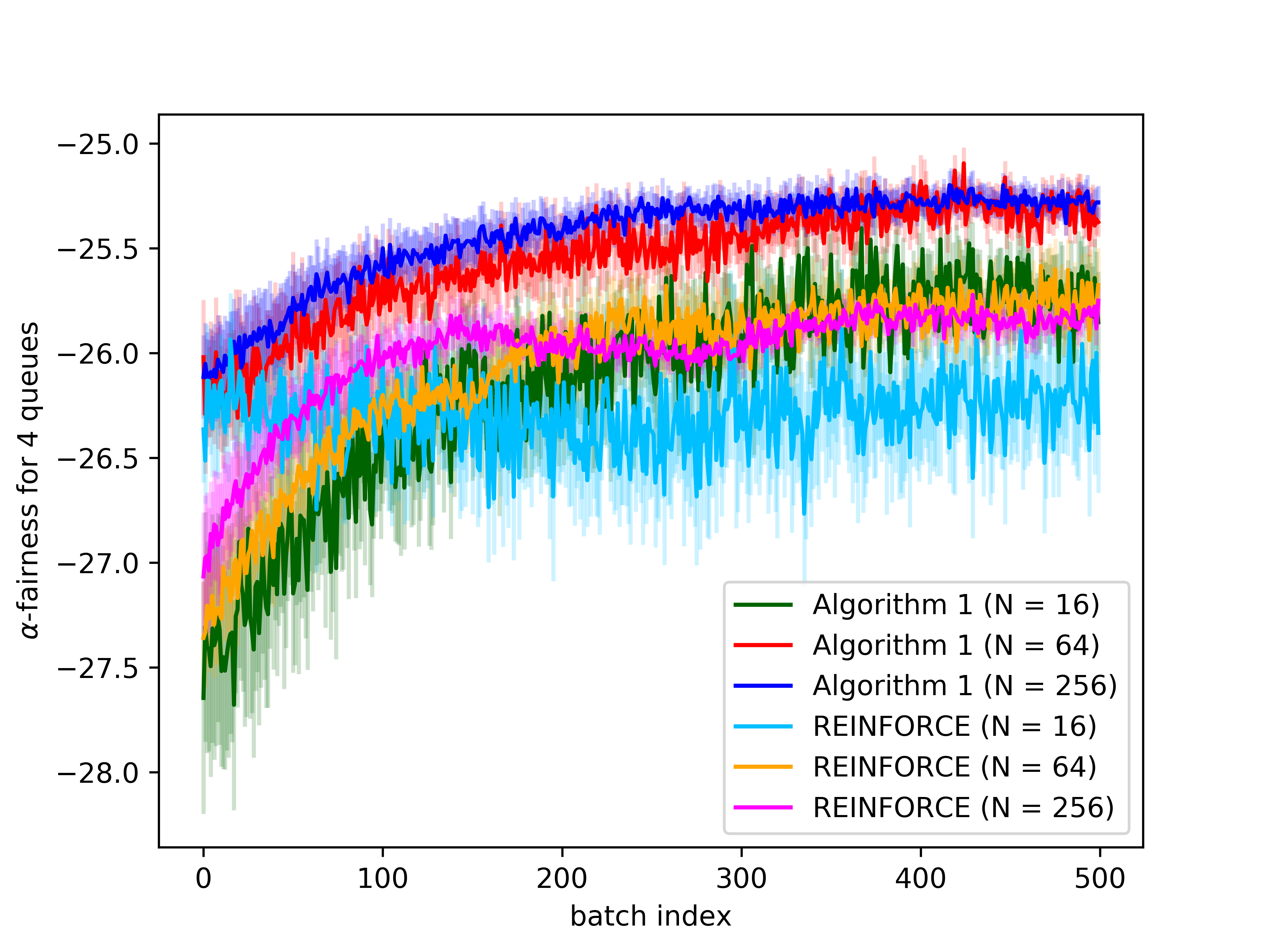}
		\label{fig:reinforce_M_4}
	}
	\subfigure[$M = 8$]{
		\includegraphics[trim=0in 0.05in 0.5in 0.45in, clip, width=0.47\textwidth]{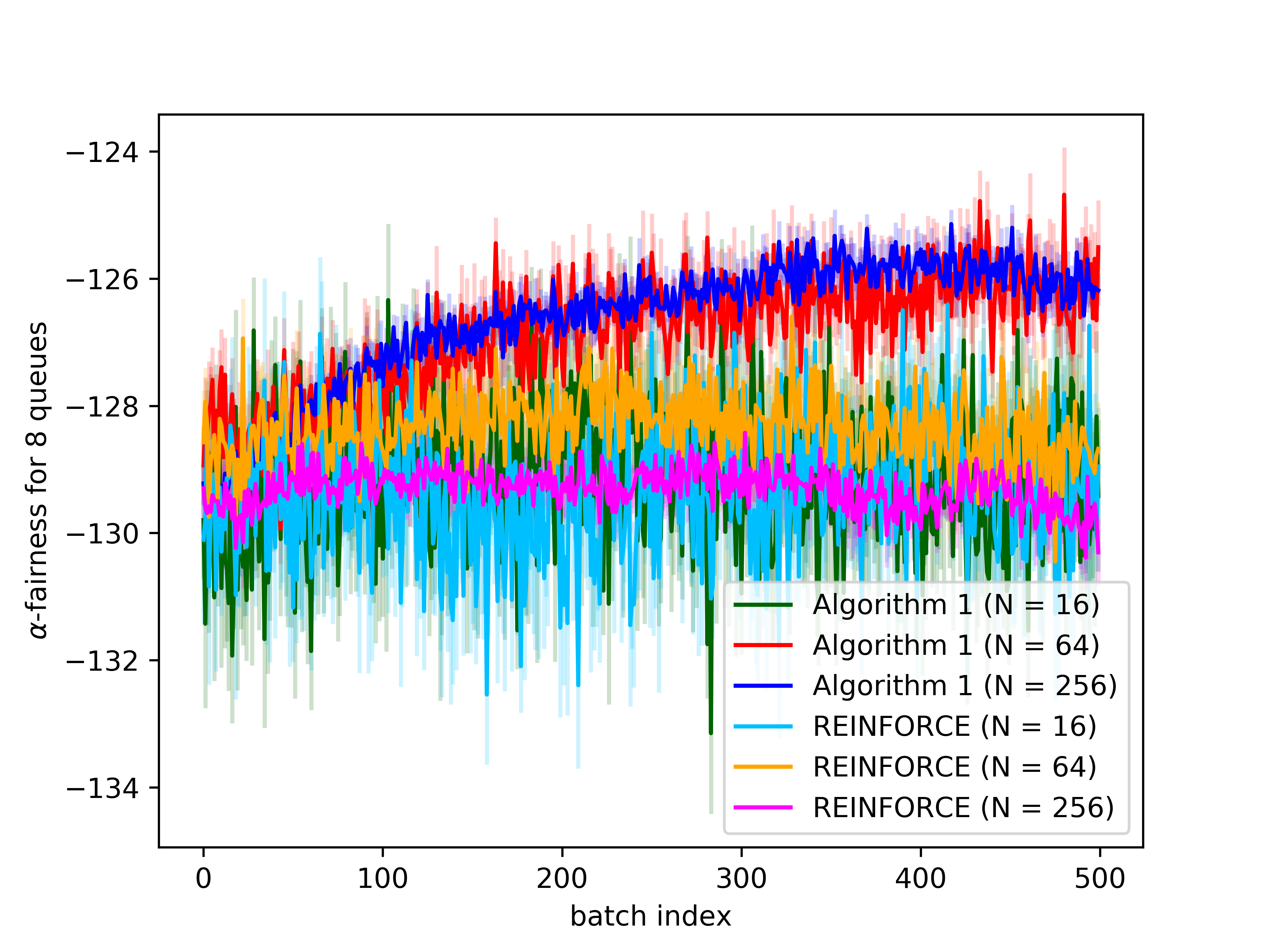}
		\label{fig:reinforce_M_8}
	}	
	\caption{ Comparison plots for the joint objective policy gradient algorithms and the REINFORCE algorithm for increasing number of queues $M$ and varied number of samples $N$ to compute gradient estimates. The REINFORCE algorithm is able to learn policies which improves the function value, but it does not achieves the policies as good as policies which our algorithm learns.}
	\label{fig:comparison_with_reinforce}
\end{figure}

    	\section{Conclusion}\label{sec_conclusion}
In this paper, we formulate a problem which  optimizes a general concave function of multiple objectives. We propose a policy-gradient based approach for the problem, where an estimator for the gradient is used.  We analyze the bias of the policy gradient estimator and show the global convergence result with a vanilla policy gradient algorithm. However, there are several limitations in the paper. Firstly, the local Lipschitz assumption for partial derivative of $f$ is a bit strong. Secondly, the convergence rates w.r.t the number of objective is $\mathcal{O}(M^4)$, while a model-based algorithm \cite{mridulmb} can achieve $\mathcal{O}(M^2)$. Further, extension of the proposed approach to evaluate the convergence rate guarantees of the Natural Policy Gradient and the variance reduced algorithms is an important future direction to reduce the sample complexity. Finally, the analysis of concave function of objectives with constraints is open for parametrized model-free setup, while has been studied for model-based setup \cite{agarwal2021concave}, for model-free tabular setup \cite{DBLP:journals/corr/abs-2109-06332}, and for parametrized model-free setup with linear function of objectives \cite{bai2022achieving}.


	\appendix	
	\onecolumn
\section{Symbol Summary}
\begin{table*}[htbp]
	{
		\centering
			\begin{tabular}{|c|c|c|}
				\hline
				Symbol & Definition & Reference  \\
				\hline
				$\mathcal{S},\mathcal{A}$  & State and Action space & Section \ref{sec_formulation}  \\
				\hline
				$\mathbb{P}$  & transition dynamics & Section \ref{sec_formulation}  \\
				\hline
				$r_m$  & reward function for $m^{th}$ objective & Section \ref{sec_formulation}  \\
				\hline
				$M$  & Number of objectives  & Section \ref{sec_formulation}  \\
				\hline
				$\gamma$  & discounted factor &  Section \ref{sec_formulation} \\
				\hline
				$\rho$  & distribution for initial state & Section \ref{sec_formulation}  \\
				\hline
				$J_m^{\pi}$  & Expected value function for $m^{th}$ objective  & Eq. \eqref{eq:expected_V}  \\
				\hline
				$V_m^{\pi}(s)$  & State value function for $m^{th}$ objective &  Eq. \eqref{eq:VandQ} \\
				\hline
				$Q_m^{\pi}(s,a)$  & State-action value function for $m^{th}$ objective &  Eq. \eqref{eq:VandQ} \\
				\hline
				$A_m^{\pi}(s,a)$  & Advantage function for $m^{th}$ objective  & Lemma \ref{lem_performance_diff} \\
				\hline
				$G$  & Lipschitz constant for log-likelihood function & Assumption \ref{ass_score}  \\
				\hline
				$B$  & smooth constant for log-likelihood function & Assumption \ref{ass_score}  \\
				\hline
				$L_f$ & Lipschitz constant for partial derivatives of function $f$ & Assumption \ref{ass_partial_grad}  \\
				\hline
				$C$ & Bound on partial derivatives of function $f$ & Lemma \ref{lem_bound_parital}  \\
				\hline
				$H$ & truncation on proposed estimator & Section \ref{sec_estimator}  \\
				\hline
				$L_J$ & smooth constant for objective function & Lemma \ref{lem_smooth}  \\
				\hline
				$\sigma^2$ & bound on variance of auxiliary estimator & Assumption \ref{ass_bounded_var}  \\
				\hline
				$\mu_F$ & positive definitive constant for Fisher information matrix & Assumption \ref{ass_pd}  \\
				\hline
				$\epsilon_{bias}$ & bias of transferred function approximation error & Assumption \ref{ass_transfer_error}  \\
				\hline
				$\eta$ & learning rate of policy gradient & Algorithm \ref{alg:PG} \\
				\hline
				$N_1,N_2$ & Number of samples for estimator & Algorithm \ref{alg:PG} \\
				\hline
				$K$ & number of iterations of policy gradient & Algorithm \ref{alg:PG} \\
				\hline
			\end{tabular}
		\caption{{ Overview of symbols defined in the paper}}
		\label{tab:symbol}
	}
\end{table*}
\section{Computation of the Gradient of Objective }\label{app_grad_computation}
\begin{equation}
	\nabla_{\theta}f(\bm{J}^{\pi_\theta})=\sum_{m=1}^{M}\frac{\partial f(\bm{J}^{\pi_\theta})}{\partial J_m^{\pi_\theta}}\nabla_\theta J_m^{\pi_\theta}
\end{equation}
Define $\tau=(s_0,a_1,s_1,a_1,s_2,a_2\cdots)$ as a trajectory, whose distribution induced by policy $\pi_\theta$ is $p(\tau\vert \theta)$ that can be expressed as
\begin{equation}\label{eq:traj_dis}
	p(\tau\vert \theta)=\rho(s_0)\prod_{t=0}^{\infty}\pi_\theta(a_t\vert s_t)P(s_{t+1}\vert s_t,a_t)
\end{equation}
Define $R_m(\tau)=\sum_{t=0}^{\infty}\gamma^tr_m(s_t,a_t)$ as the cumulative reward for $m^{th}$ objective following the trajectory $\tau$. Then, the expected return $J_m^\pi(\theta)$ can also be expressed as
\begin{equation*}
	J_m^{\pi_\theta}=\mathbf{E}_{\tau\sim p(\tau\vert\theta)}[R_m(\tau)]
\end{equation*}
and the gradient can be calculated as
\begin{equation*}
	\begin{aligned}
	\nabla_\theta J_m^\pi(\theta)&=\int_\tau R_m(\tau)p(\tau\vert \theta)d\tau =\int_{\tau}R_m(\tau)\frac{\nabla_\theta p(\tau\vert \theta)}{p(\tau\vert \theta)}p(\tau\vert \theta)d\tau\\
	&=\mathbf{E}_{\tau\sim p(\tau\vert \theta)}\big[\nabla_\theta \log p(\tau\vert\theta)R_m(\tau)\big]
	\end{aligned}
\end{equation*}
Notice that $\nabla_\theta \log p(\tau\vert \theta)$ is independent of the transition dynamics and thus 
\begin{equation}
	\nabla_\theta f(\bm{J}^{\pi_\theta})=\mathbf{E}_{\tau\sim p(\tau\vert \theta)}\bigg[\bigg(\sum_{t=0}^\infty \nabla_\theta \log \pi_\theta(a_t\vert s_t)\bigg)\bigg(\sum_{m=1}^{M}\frac{\partial f}{\partial J_m^\pi}\big(\sum_{t=0}^\infty \gamma^tr_m(s_t,a_t)\big)\bigg)\bigg]
\end{equation}
\section{Proof of Equivalence between PGT and REINFORCE Estimator in Lemma \ref{lem_grad_simplified}}\label{sec_app_equivlence}
\begin{proof}
	Notice that the difference between PGT and REINFORCE can be expressed as below
	\begin{equation}
	\sum_{t=0}^H \nabla_\theta \log(\pi_\theta(a_t^i\vert s_t^i))\bigg(\sum_{m=1}^{M}\big(\frac{\partial f}{\partial J_{m}^\pi}\bigg|_{J_m^\pi=\hat{J}_{m,H}^{\pi}}\big)\big(\sum_{h=0}^{t-1} \gamma^hr_m(s_h^i,a_h^i)\big)\bigg)
	\end{equation}
	Thus, it is sufficient to show the expectation of above equation is $\bm{0}$. Divide the trajectory $\tau_i$ into two parts $\tau_{i,1}=(s_0^i,a_0^i,\cdots,s_{t-1}^i,a_{t-1}^i)$ and $\tau_{i,2}=(s_t^i,a_t^i,\cdots)$. Then,
	\begin{equation}
	\begin{aligned}
	&\mathbf{E}_{\tau_i,\tau_j\sim p(\tau\vert \theta)}\bigg[\sum_{t=0}^H \nabla_\theta \log(\pi_\theta(a_t^i\vert s_t^i))\bigg(\sum_{m=1}^{M}\big(\frac{\partial f}{\partial J_{m}^\pi}\bigg|_{J_m^\pi=\hat{J}_{m,H}^{\pi}}\big)\big(\sum_{h=0}^{t-1} \gamma^hr_m(s_h^i,a_h^i)\big)\bigg)\bigg]\\
	&=\sum_{t=0}^H \mathbf{E}_{\tau_{i,1}}\bigg\{\mathbf{E}_{\tau_{i,2}}\bigg[\nabla_\theta \log(\pi_\theta(a_t^i\vert s_t^i))\bigg(\sum_{m=1}^{M}\mathbf{E}_{\tau_j}\big(\frac{\partial f}{\partial J_{m}^\pi}\bigg|_{J_m^\pi=\hat{J}_{m,H}^{\pi}}\big)\big(\sum_{h=0}^{t-1} \gamma^hr_m(s_h^i,a_h^i)\big)\bigg)\bigg]\bigg|\tau_{i,1}\bigg\}\\
	&=\sum_{t=0}^H \mathbf{E}_{\tau_{i,1}}\bigg\{\mathbf{E}_{\tau_{i,2}}\bigg[\nabla_\theta \log(\pi_\theta(a_t^i\vert s_t^i))\bigg]\bigg(\sum_{m=1}^{M}\mathbf{E}_{\tau_j}\big(\frac{\partial f}{\partial J_{m}^\pi}\bigg|_{J_m^\pi=\hat{J}_{m,H}^{\pi}}\big)\big(\sum_{h=0}^{t-1} \gamma^hr_m(s_h^i,a_h^i)\big)\bigg)\bigg|\tau_{i,1}\bigg\}\\
	&=\sum_{t=0}^H \mathbf{E}_{\tau_{i,1}}\bigg\{\bm{0}\cdot\bigg(\sum_{m=1}^{M}\mathbf{E}_{\tau_j}\big(\frac{\partial f}{\partial J_{m}^\pi}\bigg|_{J_m^\pi=\hat{J}_{m,H}^{\pi}}\big)\big(\sum_{h=0}^{t-1} \gamma^hr_m(s_h^i,a_h^i)\big)\bigg)\bigg|\tau_{i,1}\bigg\}=\bm{0}
	\end{aligned}
	\end{equation}
	where the first step holds because $\tau_i,\tau_j$ are dependent and the law of total expectation. The second equality holds because the summation of reward is a constant conditioned on $\tau_{i,1}$. The last step is true because 
	\begin{eqnarray}
	\mathbf{E}_{\tau_{i,2}}\bigg[\nabla_\theta \log(\pi_\theta(a_t^i\vert s_t^i))\bigg]&=&\mathbf{E}_{s_t^i}\bigg[\int_{\mathcal{A}}\nabla_\theta \log(\pi_\theta(a_t^i\vert s_t^i))\pi_\theta(a_t^i|s_t^i)da\bigg]\\
	&=&\mathbf{E}_{s_t^i}\bigg[\int_{\mathcal{A}}\nabla_\theta \pi_\theta(a_t^i|s_t^i)da\bigg]=\mathbf{E}_{s_t^i}[\nabla_\theta\mathbf{1}]=\bm{0}
	\end{eqnarray}
\end{proof}

\section{Proof for the Bias of Estimator in Eq. \eqref{eq:est_origin}}\label{sec_app_bias}
\begin{lemma}\label{lem_biased_est}
	{In general, the proposed estimator, Eq. \eqref{eq:est_origin}, is biased w.r.t. $\nabla_\theta f(\bm{J}^{\pi_\theta})$. The only exception is when the partial derivatives $\frac{\partial f}{\partial J_m^{\pi}}$ are linear w.r.t. each variable $J_n^{\pi}$ for all  $m,n\in[M]$.}
\end{lemma}
\begin{proof}
	By the law of total expectation
	\begin{equation}
		\begin{aligned}
			\mathbf{E}_{\tau_i,\tau_{j=1:N_2}}[g(\tau_i\vert \theta)]	&=\mathbf{E}_{\tau_i,\tau_{j=1:N_2}}\bigg[\sum_{t=0}^\infty \nabla_\theta \log\pi_\theta(a_t^i\vert s_t^i)\bigg(\sum_{m=1}^{M}\big(\frac{\partial f}{\partial J_m^\pi}\bigg|_{J_m^\pi=\hat{J}_{m}^{\pi}}\big)\big(\sum_{h=t}^\infty \gamma^hr_m(s_h^i,a_h^i)\big)\bigg)\bigg]\\
			&=\mathbf{E}_{\tau_i}\bigg\{\mathbf{E}_{\tau_{j=1:N_2}}\bigg[\sum_{t=0}^\infty \nabla_\theta \log\pi_\theta(a_t^i\vert s_t^i)\bigg(\sum_{m=1}^{M}\big(\frac{\partial f}{\partial J_m^\pi}\bigg|_{J_m^\pi=\hat{J}_{m}^{\pi}}\big)\big(\sum_{h=t}^\infty \gamma^hr_m(s_h^i,a_h^i)\big)\bigg)\bigg]\bigg|\tau_i\bigg\}\\
			&=\mathbf{E}_{\tau_i}\bigg\{\sum_{t=0}^\infty \nabla_\theta \log\pi_\theta(a_t^i\vert s_t^i)\bigg(\sum_{m=1}^{M}\mathbf{E}_{\tau_{j=1:N_2}}\bigg[\frac{\partial f}{\partial J_m^\pi}\bigg|_{J_m^\pi=\hat{J}_{m}^{\pi}}\bigg]\big(\sum_{h=t}^\infty \gamma^hr_m(s_h^i,a_h^i)\big)\bigg)\bigg|\tau_i\bigg\}\\
			&\overset{(*)}\neq\mathbf{E}_{\tau_i}\bigg\{\sum_{t=0}^\infty \nabla_\theta \log\pi_\theta(a_t^i\vert s_t^i)\bigg(\sum_{m=1}^{M}\frac{\partial f}{\partial J_m^\pi}\big(\sum_{h=t}^\infty \gamma^hr_m(s_h^i,a_h^i)\big)\bigg)\bigg\}\\
			&=\nabla_\theta f(J_1^\pi(s),J_2^\pi(s),\cdots,J_M^\pi(s))
		\end{aligned}
	\end{equation}
	Notice that the key step (*) holds because
	\begin{equation}
		\begin{aligned}
		\mathbf{E}_{\tau_{j=1:N_2}}\bigg[\frac{\partial f}{\partial J_m^\pi}\bigg|_{J_m^\pi=\hat{J}_{m}^{\pi}}\bigg]&=\mathbf{E}_{\tau_{j=1:N_2}}\bigg[\frac{\partial f}{\partial J_m^\pi}(\frac{1}{N_2}\sum_{j=1}^{N_2}\sum_{t=0}^{\infty}\gamma^tr_1(s_t^j,a_t^j),\cdots, \frac{1}{N_2}\sum_{j=1}^{N_2}\sum_{t=0}^{\infty}\gamma^tr_M(s_t^j,a_t^j))\bigg]\\
		&\neq \frac{\partial f}{\partial J_m^\pi}(\mathbf{E}_{\tau_{j=1:N_2}}\bigg[\frac{1}{N_2}\sum_{j=1}^{N_2}\sum_{t=0}^{\infty}\gamma^tr_1(s_t^j,a_t^j)\bigg],\cdots,\mathbf{E}_{\tau_{j=1:N_2}}\bigg[\frac{1}{N_2}\sum_{j=1}^{N_2}\sum_{t=0}^{\infty}\gamma^tr_M(s_t^j,a_t^j)\bigg])\\
		&=\frac{\partial f}{\partial J_m^\pi}(J_1^\pi,\cdots,J_M^\pi)
		\end{aligned}
	\end{equation}
	Eq. 39 cannot hold with an equality except when the partial derivatives are linear. However, this doesn't hold for any general concave function.
\end{proof}

\section{Bounding the Bias for the Proposed Estimator}\label{sec_app_bound_bias}
\subsection{Proof for Lemma \ref{lem_bound_bias1}}
\begin{proof}
	By the triangle inequality, Assumptions \ref{ass_bound_reward} and \ref{ass_score}, we have
	\begin{equation}\label{eq:bound_bias1}
		\begin{split}
		\Vert g(\tau_i^H,\tau_j^H\vert \theta)-\tilde{g}(\tau_i^H,\tau_j^H\vert \theta)\Vert&=\bigg\Vert\sum_{t=0}^{H-1} \nabla_\theta \log\pi_\theta(a_t^i\vert s_t^i)\bigg(\sum_{m=1}^{M}\big(\frac{\partial f}{\partial J_m^\pi}\bigg|_{J_m^\pi=\hat{J}_{m,H}^\pi}-\frac{\partial f}{\partial J_m^\pi}\bigg|_{J_m^\pi=J_{m,H}^\pi}\big)\\
		&\big(\sum_{h=t}^{H-1} \gamma^hr_m(s_h^i,a_h^i)\big)\bigg)\bigg\Vert\\
		&\leq\frac{G}{1-\gamma}\bigg|\sum_{t=0}^{H-1}(\gamma^t-\gamma^H) \bigg(\sum_{m=1}^{M}\big(\frac{\partial f}{\partial J_m^\pi}\bigg|_{J_m^\pi=\hat{J}_{m,H}^\pi}-\frac{\partial f}{\partial J_m^\pi}\bigg|_{J_m^\pi=J_{m,H}^\pi}\big)\bigg)\bigg|\\
		&\leq G\frac{1-\gamma^H-H\gamma^H(1-\gamma)}{(1-\gamma)^2} \sum_{m=1}^{M}\bigg|\frac{\partial f(\hat{J}_{m,H}^\pi)}{\partial J_m^\pi}-\frac{\partial f(J_{m,H}^\pi)}{\partial J_m^\pi}\bigg|\\
		&\leq GML_f\frac{1-\gamma^H-H\gamma^H(1-\gamma)}{(1-\gamma)^2} \Vert \hat{\bm{J}}_H^\pi-\bm{J}_H^\pi\Vert\\
		\end{split}
	\end{equation}
	where the last step follows from Assumption \ref{ass_partial_grad}. Moreover, an entry  in the difference $\hat{\bm{J}}_H^\pi-\bm{J}_H^\pi$ can be bounded as 
	\begin{equation}\label{eq:bound_JmH}
	\begin{split}
	\vert \hat{J}_{m,H}^\pi-J_{m,H}^\pi\vert&=\bigg|\frac{1}{N_2}\sum_{j=1}^{N_2}\sum_{t=0}^{H-1}\gamma^tr_m(s_t,a_t)-\mathbf{E}\big[\sum_{t=0}^{H-1}\gamma^tr_m(s_t,a_t)\big]\bigg|\\
	&\leq \sum_{t=0}^{H-1}\gamma^t\bigg|\frac{1}{N_2}\sum_{j=1}^{N_2}r_m(s_t,a_t)-\mathbf{E}[r_m(s_t,a_t)]\bigg|
	\end{split}
	\end{equation}
	By Hoeffding Lemma, if we have $N_2\geq \frac{M(1-\gamma^H)^2}{2(1-\gamma)^2\epsilon'^2}\log(\frac{2MH}{p})$, then
	\begin{equation}
	\begin{split}
	P\bigg(\bigg|\frac{1}{N_2}\sum_{j=1}^{N_2}r_m(s_t,a_t)-\mathbf{E}[r_m(s_t,a_t)]\bigg|\geq \frac{(1-\gamma)\epsilon'}{(1-\gamma^H)\sqrt{M}}\bigg)&\leq 2\exp(-\frac{2N_2^2\frac{(1-\gamma)^2\epsilon'^2}{(1-\gamma^H)^2 M}}{\sum_{j=1}^{N_2}(1-0)^2})\leq\frac{p}{MH}
	\end{split}
	\end{equation}
	Finally, by using an union bound, with probability at least $1-p$, we have
	\begin{equation}\label{eq:hoeffding_rm}
	\bigg|\frac{1}{N_2}\sum_{j=1}^{N_2}r_m(s_t,a_t)-\mathbf{E}[r_m(s_t,a_t)]\bigg|\leq \frac{(1-\gamma)\epsilon'}{(1-\gamma^H)\sqrt{M}}\quad \forall m\in[M],\forall t\in[0,H-1]
	\end{equation}
	Substituting Eq. \eqref{eq:hoeffding_rm} back into \eqref{eq:bound_JmH}, we have $\vert \hat{J}_{m,H}^\pi-J_{m,H}^\pi\vert\leq\frac{\epsilon'}{\sqrt{M}}$ and thus $\Vert \bm{J}_H^\pi-\hat{\bm{J}}_H^\pi\Vert_2\leq \epsilon'$, which gives the result in the statement of the Lemma.	
\end{proof}

\subsection{Proof for Lemma \ref{lem_bound_bias2}}
\begin{proof}
	Similar to Eq. \eqref{eq:bound_bias1}, we have
	\begin{equation}\label{eq:bound_bias2}
	\Vert \tilde{g}(\tau_i^H,\tau_j^H\vert \theta)-\tilde{g}(\tau_i^H,\tau_j\vert \theta)\Vert\leq GML_f\frac{1-\gamma^H-H\gamma^H(1-\gamma)}{(1-\gamma)^2} \Vert \bm{J}_H^\pi-\bm{J}^\pi\Vert\\
	\end{equation}
	By triangle inequality, the element of $\bm{J}_H^\pi-\bm{J}^\pi$ can be bounded by
	\begin{equation}\label{eq:bound_Jm}
	\begin{split}
	\vert J_{m,H}^\pi-J_m^\pi\vert&\leq \bigg|\mathbf{E}\big[\sum_{t=0}^{\infty}\gamma^tr_m(s_t,a_t)\big]-\mathbf{E}\big[\sum_{t=0}^{H-1}\gamma^tr_m(s_t,a_t)\big]\bigg|\\
	&\leq\sum_{t=H}^{\infty}\gamma^t\bigg|\mathbf{E}[r_m(s_t,a_t)]\bigg|\leq \frac{\gamma^H}{1-\gamma}
	\end{split}
	\end{equation}
	where the last step holds by Assumption \ref{ass_bound_reward}. Substituting Eq \eqref{eq:bound_Jm} back into \eqref{eq:bound_bias2} gives the result in the statement of the Lemma.
\end{proof}

\subsection{Proof for Lemma \ref{lem_bound_bias3}}
\begin{proof}
	By the triangle inequality,
	\begin{equation}
	\begin{aligned}
	&\Vert \tilde{g}(\tau_i^H,\tau_j\vert \theta)-g(\tau_i,\tau_j\vert \theta)\Vert\\
	&=\Vert\sum_{t=0}^\infty \nabla_\theta \log\pi_\theta(a_t^i\vert s_t^i)\bigg(\sum_{m=1}^{M}\frac{\partial f}{\partial J_m^\pi}\big(\sum_{h=t}^{\infty} \gamma^hr_m(s_h^i,a_h^i)\big)\bigg)\nonumber\\
	& -\sum_{t=0}^{H-1} \nabla_\theta \log\pi_\theta(a_t^i\vert s_t^i)\bigg(\sum_{m=1}^{M}\frac{\partial f}{\partial J_m^\pi}\big(\sum_{h=t}^{\infty} \gamma^hr_m(s_h^i,a_h^i)\big)\bigg)\\
	&+\sum_{t=0}^{H-1} \nabla_\theta \log\pi_\theta(a_t^i\vert s_t^i)\bigg(\sum_{m=1}^{M}\frac{\partial f}{\partial J_m^\pi}\big(\sum_{h=t}^{\infty} \gamma^hr_m(s_h^i,a_h^i)\big)\bigg)\nonumber\\
	&-\sum_{t=0}^{H-1} \nabla_\theta \log\pi_\theta(a_t^i\vert s_t^i)\bigg(\sum_{m=1}^{M}\frac{\partial f}{\partial J_m^\pi}\big(\sum_{h=t}^{H-1} \gamma^hr_m(s_h^i,a_h^i)\big)\bigg)\Vert\\
	&\leq \Vert\sum_{t=H}^\infty \nabla_\theta \log\pi_\theta(a_t^i\vert s_t^i)\bigg(\sum_{m=1}^{M}\frac{\partial f}{\partial J_m^\pi}\big(\sum_{h=t}^{\infty} \gamma^hr_m(s_h^i,a_h^i)\big)\bigg)\Vert\\
	&+\Vert\sum_{t=0}^{H-1} \nabla_\theta \log\pi_\theta(a_t^i\vert s_t^i)\bigg(\sum_{m=1}^{M}\frac{\partial f}{\partial J_m^\pi}\big(\sum_{h=H}^{\infty} \gamma^hr_m(s_h^i,a_h^i)\big)\bigg)\Vert\\
	&\leq \frac{MGC\gamma^H}{(1-\gamma)^2}+\frac{MGCH\gamma^H}{(1-\gamma)}=MGC\frac{\gamma^H(1+H(1-\gamma))}{(1-\gamma)^2}
	\end{aligned}
	\end{equation}
	where the last inequality holds by  Lemma \ref{lem_bound_parital} and Assumption \ref{ass_score}.
\end{proof}
\section{Proof for Properties of the Objective Function}\label{sec_app_property}
\subsection{Proof for Lemma \ref{lem_smooth}}
\begin{proof}
	In order to show the smoothness, it is sufficient to bound  $\Vert\nabla^2_\theta f(\bm{J}^{\pi_\theta})\Vert$ and $\Vert\nabla^2_\theta f(\bm{J}_H^{\pi_\theta})\Vert$. By  Eq. \eqref{eq:gradient_origin}, we have
	\begin{equation}
	\begin{split}
	\Vert\nabla_\theta^2 f(\bm{J}^{\pi_\theta})\Vert&=\Vert\mathbf{E}_{\tau\sim p(\tau\vert \theta)}\bigg[\sum_{t=0}^{\infty}\nabla_\theta^2\log\pi_\theta(a_t\vert s_t)\bigg(\sum_{m=1}^{M}\frac{\partial f}{\partial J_m^\pi}\big(\sum_{h=t}^{\infty}\gamma^hr_m(s_h,a_h)\big)\bigg)\bigg]\Vert\\
	&\leq \frac{MC}{(1-\gamma)}\sum_{t=0}^{\infty}\gamma^t\Vert\nabla_\theta^2\log\pi_\theta(a_t\vert s_t)\Vert\leq \frac{MCB}{(1-\gamma)^2}
	\end{split}
	\end{equation}
	where the last inequality holds by the Assumption \ref{ass_score}. The smoothness property for the truncated version $f(\bm{J}_H^{\pi_\theta})$ can be proved similarly.
\end{proof}
\subsection{Proof for Lemma \ref{lem_bound_truncated}}
\begin{proof}
	Notice that $\tilde{g}(\tau_i,\tau_j\vert \theta)$ is an unbiased estimator for $\nabla_\theta f(\bm{J}^{\pi_\theta})$. Moreover, $\tilde{g}(\tau_i^H,\tau_j^H)$ is an unbiased estimator for $\nabla_\theta f(\bm{J}_H^{\pi_\theta})$. Thus,
	\begin{equation}
		\begin{split}
		\Vert\nabla_\theta f(\bm{J}^{\pi_\theta})-\nabla_\theta f(\bm{J}_H^{\pi_\theta})\Vert&\overset{(a)}=\Vert\mathbf{E}[\tilde{g}(\tau_i,\tau_j\vert \theta)-\tilde{g}(\tau_i^H,\tau_j^H\vert \theta)]\Vert\leq \mathbf{E}\Vert\tilde{g}(\tau_i,\tau_j\vert \theta)-\tilde{g}(\tau_i^H,\tau_j^H\vert \theta)\Vert\\
		&\overset{(b)}\leq \mathbf{E}\Vert\tilde{g}(\tau_i,\tau_j\vert \theta)-\tilde{g}(\tau_i^H,\tau_j\vert \theta)\Vert+\mathbf{E}\Vert\tilde{g}(\tau_i^H,\tau_j\vert \theta)-\tilde{g}(\tau_i^H,\tau_j^H\vert \theta)\Vert\\
		&\overset{(c)}\leq M^{3/2}GL_f\frac{1-\gamma^H-H\gamma^H(1-\gamma)}{(1-\gamma)^3}\gamma^H+ MGC\frac{\gamma^H[1+H(1-\gamma)]}{(1-\gamma)^2}
		\end{split}
	\end{equation}
	where the step (a) and (b) hold by the triangle inequality. Step (c) holds by the Lemma \ref{lem_bound_bias2} and \ref{lem_bound_bias3}
\end{proof}	

\subsection{Proof for Lemma \ref{lem_performance_diff}}
\begin{proof}
	By the concavity of the function $f$, we have
	\begin{equation}
	\begin{aligned}
	f(\bm{J}^{\pi_\theta})&\leq  f(\bm{J}^{\pi_{\theta'}})+\nabla_{\bm{J}^{\pi_{\theta'}}} f(\bm{J}^{\pi_{\theta'}})^T(\bm{J}^{\pi_\theta}-\bm{J}^{\pi_{\theta'}})\\
	&=f(\bm{J}^{\pi_{\theta'}})+\sum_{m=1}^{M}\frac{\partial f(\bm{J}^{\pi_{\theta'}})}{\partial J_m^{\pi_{\theta'}}}(J_m^{\pi_\theta}-J_m^{\pi_{\theta'}})\\
	&=f(\bm{J}^{\pi_{\theta'}})+\sum_{m=1}^{M}\frac{\partial f(\bm{J}^{\pi_{\theta'}})}{\partial J_m^{\pi_{\theta'}}}\frac{1}{1-\gamma}\mathbf{E}_{s\sim d_\rho^{\pi_\theta}}\mathbf{E}_{a\sim\pi_\theta(\cdot\vert s)}\big[A_m^{\pi_{\theta'}}(s,a)\big]
	\end{aligned}
	\end{equation}
	where the last step comes from the policy gradient theorem \cite{sutton2000} for the standard reinforcement learning. Finally,  we get the desired result by rearranging terms.
\end{proof}
\section{Proof of Lemma \ref{lem_framework}}\label{sec_app_framwork}
\begin{proof}
	Starting with the definition of KL divergence,
	\begin{equation}
	\begin{aligned}
	&\mathbf{E}_{s\sim d_\rho^{\pi^*}}[KL(\pi^*(\cdot\vert s)\Vert\pi_{\theta^k}(\cdot\vert s))-KL(\pi^*(\cdot\vert s)\Vert\pi_{\theta^{k+1}}(\cdot\vert s))]\\
	=&\mathbf{E}_{s\sim d_\rho^{\pi^*}}\mathbf{E}_{a\sim\pi^*(\cdot\vert s)}\bigg[\log\frac{\pi_{\theta^{k+1}(a\vert s)}}{\pi_{\theta^k}(a\vert s)}\bigg]\\
	\overset{(a)}\geq&\mathbf{E}_{s\sim d_\rho^{\pi^*}}\mathbf{E}_{a\sim\pi^*(\cdot\vert s)}[\nabla_\theta\log\pi_{\theta^k}(a\vert s)\cdot(\theta^{k+1}-\theta^k)]-\frac{B}{2}\Vert\theta^{k+1}-\theta^k\Vert^2\\
	=&\eta\mathbf{E}_{s\sim d_\rho^{\pi^*}}\mathbf{E}_{a\sim\pi^*(\cdot\vert s)}[\nabla_\theta\log\pi_{\theta^k}(a\vert s)\cdot\omega^k]-\frac{B\eta^2}{2}\Vert\omega^k\Vert^2\\
	=&\eta\mathbf{E}_{s\sim d_\rho^{\pi^*}}\mathbf{E}_{a\sim\pi^*(\cdot\vert s)}[\nabla_\theta\log\pi_{\theta^k}(a\vert s)\cdot\omega_*^k]+\eta\mathbf{E}_{s\sim d_\rho^{\pi^*}}\mathbf{E}_{a\sim\pi^*(\cdot\vert s)}[\nabla_\theta\log\pi_{\theta^k}(a\vert s)\cdot(\omega^k-\omega_*^k)]-\frac{B\eta^2}{2}\Vert\omega^k\Vert^2\\
	=&\eta[f(\bm{J}^{\pi^*})-f(\bm{J}^{\pi_{\theta^k}})]+\eta\mathbf{E}_{s\sim d_\rho^{\pi^*}}\mathbf{E}_{a\sim\pi^*(\cdot\vert s)}[\nabla_\theta\log\pi_{\theta^k}(a\vert s)\cdot\omega_*^k]-\eta[f(\bm{J}^{\pi^*})-f(\bm{J}^{\pi_{\theta^k}})]\\
	&+\eta\mathbf{E}_{s\sim d_\rho^{\pi^*}}\mathbf{E}_{a\sim\pi^*(\cdot\vert s)}[\nabla_\theta\log\pi_{\theta^k}(a\vert s)\cdot(\omega^k-\omega_*^k)]-\frac{B\eta^2}{2}\Vert\omega^k\Vert^2\\		\overset{(b)}=&\eta[f(\bm{J}^{\pi^*})-f(\bm{J}^{\pi_{\theta^k}})]+\frac{\eta}{1-\gamma}\mathbf{E}_{s\sim d_\rho^{\pi^*}}\mathbf{E}_{a\sim\pi^*(\cdot\vert s)}\bigg[\nabla_\theta\log\pi_{\theta^k}(a\vert s)\cdot(1-\gamma)\omega_*^k-\sum_{m=1}^{B}\frac{\partial f(\bm{J}^{\pi_{\theta^k}})}{\partial J_m^{\pi_{\theta^k}}}A_m^{\pi_{\theta^k}}(s,a)\bigg]\\
	&+\eta\mathbf{E}_{s\sim d_\rho^{\pi^*}}\mathbf{E}_{a\sim\pi^*(\cdot\vert s)}[\nabla_\theta\log\pi_{\theta^k}(a\vert s)\cdot(\omega^k-\omega_*^k)]-\frac{B\eta^2}{2}\Vert\omega^k\Vert^2\\
	\overset{(c)}\geq&\eta[f(\bm{J}^{\pi^*})-f(\bm{J}^{\pi_{\theta^k}})]\\
	&-\frac{\eta}{1-\gamma}\sqrt{\mathbf{E}_{s\sim d_\rho^{\pi^*}}\mathbf{E}_{a\sim\pi^*(\cdot\vert s)}\bigg[\bigg(\nabla_\theta\log\pi_{\theta^k}(a\vert s)\cdot(1-\gamma)\omega_*^k-\sum_{m=1}^{B}\frac{\partial f(\bm{J}^{\pi_{\theta^k}})}{\partial J_m^{\pi_{\theta^k}}}A_m^{\pi_{\theta^k}}(s,a)\bigg)^2\bigg]}\\
	&-\eta\mathbf{E}_{s\sim d_\rho^{\pi^*}}\mathbf{E}_{a\sim\pi^*(\cdot\vert s)}\Vert\nabla_\theta\log\pi_{\theta^k}(a\vert s)\Vert_2\Vert(\omega^k-\omega_*^k)\Vert-\frac{B\eta^2}{2}\Vert\omega^k\Vert^2\\
	\overset{(d)}\geq&\eta[f(\bm{J}^{\pi^*})-f(\bm{J}^{\pi_{\theta^k}})]-\frac{\eta\sqrt{\epsilon_{bias}}}{1-\gamma}-\eta G\Vert(\omega^k-\omega_*^k)\Vert-\frac{B\eta^2}{2}\Vert\omega^k\Vert^2\\
	\end{aligned}	
	\end{equation}
	where the step (a) holds by Assumption \ref{ass_score} and step (b) holds by Lemma \ref{lem_performance_diff}. Step (c) uses the convexity of the function $f(x)=x^2$. Finally, step (d) comes from the Assumption \ref{ass_transfer_error}. Rearranging items, we have
	\begin{equation}
	\begin{split}
	f(\bm{J}^{\pi^*})-f(\bm{J}^{\pi_{\theta^k}})&\leq \frac{\sqrt{\epsilon_{bias}}}{1-\gamma}+ G\Vert(\omega^k-\omega_*^k)\Vert+\frac{B\eta}{2}\Vert\omega^k\Vert^2\\
	&+\frac{1}{\eta}\mathbf{E}_{s\sim d_\rho^{\pi^*}}[KL(\pi^*(\cdot\vert s)\Vert\pi_{\theta^k}(\cdot\vert s))-KL(\pi^*(\cdot\vert s)\Vert\pi_{\theta^{k+1}}(\cdot\vert s))]
	\end{split}
	\end{equation}
	Summing from $k=0$ to $K-1$ and dividing by $K$, we get the desired result.
\end{proof}
	\section{Proof for Theorem \ref{thm1}}\label{sec_app_thm}
	In this part, we prove the Theorem \ref{thm1} by bounding the three terms on the right hand side of Eq. \eqref{eq:general_bound}. These terms are:  the difference between the update direction $\frac{G}{K}\sum_{k=0}^{K-1}\Vert(\omega^k-\omega_*^k)\Vert$, norm of estimated gradient $\frac{M\eta}{2K}\sum_{k=0}^{K-1}\Vert\omega^k\Vert^2$, and the term about KL divergence $\frac{1}{\eta K}\mathbf{E}_{s\sim d_\rho^{\pi^*}}[KL(\pi^*(\cdot\vert s)\Vert\pi_{\theta^0}(\cdot\vert s))]$
	
	\subsection{Bounding the Difference Between the Update Directions}
	Recall the estimated policy gradient update direction is
	\begin{equation}\label{eq:def_omega^k}
	\omega^k=\frac{1}{N_1}\sum_{i=1}^{N_1}g(\tau_i^H,\tau_j^H\vert \theta)
	\end{equation}
	and the true natural policy gradient update direction is
	\begin{equation}
	\omega_*^k=F_\rho(\theta_k)^{\dagger}\nabla_\theta f(\bm{J}^{\pi_\theta})
	\end{equation}
	We define an auxiliary update direction as
	\begin{equation}\label{eq:def_tilde_omega^k}
		\tilde{\omega}^k=\frac{1}{N_1}\sum_{i=1}^{N_1}\tilde{g}(\tau_i^H,\tau_j^H\vert \theta)
	\end{equation}
	Thus, we can decompose the difference as 
	\begin{equation}
	\begin{split}
	&\bigg(\frac{1}{K}\sum_{k=0}^{K-1}\mathbf{E}\Vert\omega^k-\omega_*^k\Vert\bigg)^2\leq\frac{1}{K}\sum_{k=0}^{K-1}\bigg(\mathbf{E}\Vert\omega^k-\omega_*^k\Vert\bigg)^2\leq \frac{1}{K}\sum_{k=0}^{K-1}\mathbf{E}\bigg[\Vert\omega_k-\omega_*^k\Vert^2\bigg]\\
	&=\frac{1}{K}\sum_{k=0}^{K-1}\mathbf{E}\bigg[\Vert(\omega^k-\tilde{\omega}^k)+(\tilde{\omega}^k-\nabla_\theta f(\bm{J}_H^{\pi_\theta}))+(\nabla_\theta f(\bm{J}_H^{\pi_\theta})-\nabla_\theta f(\bm{J}^{\pi_\theta}))+(\nabla_\theta f(\bm{J}^{\pi_\theta})-F_\rho(\theta^k)^\dagger\nabla_\theta f(\bm{J}^{\pi_\theta}))\Vert^2\bigg]\\
	&\leq \frac{4}{K}\sum_{k=0}^{K-1}\mathbf{E}\bigg[\Vert\omega^k-\tilde{\omega}^k\Vert^2\bigg]+\frac{4}{K}\sum_{k=0}^{K-1}\mathbf{E}\bigg[\Vert\tilde{\omega}^k-\nabla_\theta f(\bm{J}_H^{\pi_\theta})\Vert^2\bigg]+\frac{4}{K}\sum_{k=0}^{K-1}\mathbf{E}\bigg[\Vert\nabla_\theta f(\bm{J}^{\pi_\theta})-\nabla_\theta f(\bm{J}_H^{\pi_\theta})\Vert^2\bigg]\\
	&+\frac{4}{K}\sum_{k=0}^{K-1}\mathbf{E}\bigg[\Vert\nabla_\theta f(\bm{J}^{\pi_\theta})-F_\rho(\theta^k)^\dagger\nabla_\theta f(\bm{J}^{\pi_\theta})\Vert^2\bigg]
	\end{split}
	\end{equation}
The different terms in the above are bounded as follows:

	\begin{itemize}
		\item Bounding $\mathbf{E}\bigg[\Vert\omega^k-\tilde{\omega}^k\Vert^2\bigg]$:  By  Lemma \ref{lem_bound_bias1}, with $N_2$ large enough, for any $\tau_i$ and $\theta$, we have
		\begin{equation}
			\Vert 		g(\tau_i^H\vert\theta)-\tilde{g}(\tau_i^H\vert\theta)\Vert\leq MGL_f\frac{1-\gamma^H-H\gamma^H(1-\gamma)}{(1-\gamma)^2}\epsilon'
		\end{equation}
		Thus,
		\begin{equation}
		\begin{split}
		\Vert\omega^k-\tilde{\omega}^k\Vert&=\Vert\frac{1}{N_1}\sum_{i=1}^{N_1}(g(\tau_i^H\vert \theta^k)-\tilde{g}(\tau_i^H\vert \theta^k))\Vert\leq \frac{1}{N_1}\sum_{i=1}^{N_1}\Vert(g(\tau_i^H\vert \theta^k)-\tilde{g}(\tau_i^H\vert \theta^k))\Vert\\
		&\leq MGL_f\frac{1-\gamma^H-H\gamma^H(1-\gamma)}{(1-\gamma)^2}\epsilon'\leq \frac{MGL_f}{(1-\gamma)^2}\epsilon'
		\end{split}
		\end{equation}
		Thus,
		\begin{equation}
			\mathbf{E}\bigg[\Vert\omega^k-\tilde{\omega}^k\Vert^2\bigg]\leq \frac{M^2G^2L_f^2}{(1-\gamma)^4}\epsilon'^2
		\end{equation}
		\item Bounding $\mathbf{E}\bigg[\Vert\tilde{\omega}^k-\nabla_\theta f(\bm{J}_H^{\pi_\theta})\Vert^2\bigg]$: 	Notice that $\tilde{g}(\tau^H\vert\theta)$ is an unbiased estimator for $\nabla_\theta f(\bm{J}_H^{\pi_\theta})$ and thus by  Assumption \ref{ass_bounded_var}, we have $
		\mathbf{E}\bigg[\Vert\omega^k-\tilde{\omega}^k\Vert^2\bigg]\leq \frac{\sigma^2}{N_1}$
		\item Bounding $\mathbf{E}\bigg[\Vert\nabla_\theta f(\bm{J}^{\pi_\theta})-\nabla_\theta f(\bm{J}_H^{\pi_\theta})\Vert^2\bigg]$: 
		By Lemma \ref{lem_bound_truncated}, we have
		\begin{equation}
			\mathbf{E}\bigg[\Vert\nabla_\theta f(\bm{J}^{\pi_\theta})-\nabla_\theta f(\bm{J}_H^{\pi_\theta})\Vert^2\bigg] \leq \frac{M^2G^2\gamma^{2H}}{(1-\gamma)^4}\bigg[ \sqrt{M}L_f+C[1+H(1-\gamma)]\bigg]^2
		\end{equation}
		\item Bounding $\mathbf{E}\bigg[\Vert\nabla_\theta f(\bm{J}^{\pi_\theta})-F_\rho(\theta^k)^\dagger\nabla_\theta f(\bm{J}^{\pi_\theta})\Vert^2\bigg]$:	By  Assumption \ref{ass_pd}, we have
		\begin{equation}
			\begin{split}
			&\mathbf{E}\bigg[\Vert\nabla_\theta f(\bm{J}^{\pi_\theta})-F_\rho(\theta^k)^\dagger\nabla_\theta f(\bm{J}^{\pi_\theta})\Vert^2\bigg]\leq (1+\frac{1}{\mu_F})^2\mathbf{E}[\Vert \nabla_\theta f(\bm{J}^{\pi_k})\Vert^2]\\
			&\leq (1+\frac{1}{\mu_F})^2\bigg(2\mathbf{E}[\Vert \nabla_\theta f(\bm{J}_H^{\pi_k})\Vert^2]+2\mathbf{E}[\Vert \nabla_\theta f(\bm{J}^{\pi_k})-\nabla_\theta f(\bm{J}_H^{\pi_k})\Vert^2]\bigg)\\
			&\leq (1+\frac{1}{\mu_F})^2\bigg(2\mathbf{E}[\Vert \nabla_\theta f(\bm{J}^{\pi_k})]+\frac{2M^2G^2\gamma^{2H}}{(1-\gamma)^4}\bigg[ \sqrt{M}L_f+C[1+H(1-\gamma)]\bigg]^2\bigg)
		\end{split}
		\end{equation}
	\end{itemize}
	Finally, we obtain the bound
	\begin{equation}
	\begin{split}
	&\bigg(\frac{1}{K}\sum_{k=0}^{K-1}\mathbf{E}\Vert\omega^k-\omega_*^k\Vert\bigg)^2\leq 4\frac{M^2G^2L_f^2}{(1-\gamma)^4}\epsilon'^2+4\frac{\sigma^2}{N_1}+4\frac{M^2G^2\gamma^{2H}}{(1-\gamma)^4}\bigg[ \sqrt{M}L_f+C[1+H(1-\gamma)]\bigg]^2\\
	&+4(1+\frac{1}{\mu_F})^2\bigg(\frac{2}{K}\sum_{k=0}^{K-1}\mathbf{E}[\Vert \nabla_\theta f(\bm{J}^{\pi_k})]+\frac{2M^2G^2\gamma^{2H}}{(1-\gamma)^4}\bigg[ \sqrt{M}L_f+C[1+H(1-\gamma)]\bigg]^2\bigg)\\
	&\overset{(a)}=(1+2(1+\frac{1}{\mu_F})^2)4\frac{M^2G^2\gamma^{2H}}{(1-\gamma)^4}\bigg[ \sqrt{M}L_f+C[1+H(1-\gamma)]\bigg]^2+4\frac{M^2G^2L_f^2}{(1-\gamma)^4}\epsilon'^2+4\frac{\sigma^2}{N_1}\\
	&+8(1+\frac{1}{\mu_F})^2\frac{\frac{\mathbf{E}[f(\bm{J}_H(\theta^K))-f(\bm{J}_H(\theta^0))]}{K}+(\eta+2L_J\eta^2)[\frac{M^2G^2L_f^2}{(1-\gamma)^4}\epsilon'^2+\frac{\sigma^2}{N_1}]}{\frac{\eta}{2}-L_J\eta^2}\\
	&=(1+2(1+\frac{1}{\mu_F})^2)4\frac{M^2G^2\gamma^{2H}}{(1-\gamma)^4}\bigg[ \sqrt{M}L_f+C[1+H(1-\gamma)]\bigg]^2+(1+6(1+\frac{1}{\mu_F})^2)4\frac{M^2G^2L_f^2}{(1-\gamma)^4}\epsilon'^2\\
	&+(1+6(1+\frac{1}{\mu_F})^2)4\frac{\sigma^2}{N_1}+128(1+\frac{1}{\mu_F})^2L_J\frac{\mathbf{E}[f(\bm{J}_H(\theta^K))-f(\bm{J}_H(\theta^0))]}{K}
	\end{split}
	\end{equation}
	where the step (a) requires the first-order stationary property Eq. \eqref{eq:first_order} and it is proved in the Lemma \ref{lem_fisrt_order} in the Appendix \ref{sec_app_first_order}. Given the fixed $\epsilon$, choose the value for $H,\epsilon',N_1,K$ as follows,
	\begin{equation}
	\frac{1}{4}\bigg(\frac{\epsilon^2}{3G^2}\bigg)\geq (1+2(1+\frac{1}{\mu_F})^2)\frac{4M^2G^2\gamma^{2H}}{(1-\gamma)^4}\bigg[ \sqrt{M}L_f+C[1+H(1-\gamma)]\bigg]^2
	\end{equation}
	\begin{equation}
		\epsilon'^2\leq\frac{1}{4(1+6(1+\frac{1}{\mu_F})^2)}\frac{(1-\gamma)^4}{M^2G^2L_f^2}\cdot\frac{1}{4}\bigg(\frac{\epsilon^2}{3G^2}\bigg)
	\end{equation}
	\begin{equation}
		N_1\geq\frac{(1+6(1+\frac{1}{\mu_F})^2)4\sigma^2}{\frac{1}{4}\bigg(\frac{\epsilon^2}{3G^2}\bigg)}
	\end{equation}
	\begin{equation}
		K\geq\frac{128(1+\frac{1}{\mu_F})^2L_J\mathbf{E}[f(\bm{J}_H(\theta^K))-f(\bm{J}_H(\theta^0))]}{\frac{1}{4}\bigg(\frac{\epsilon^2}{3G^2}\bigg)}
	\end{equation}
	then we have
	\begin{equation}
	\frac{G}{K}\sum_{k=0}^{K-1}\mathbf{E}[\Vert\omega^k-\omega_*^k]\Vert\leq \frac{\epsilon}{3}
	\end{equation}
	Given the choice of $H,N_1,\epsilon',K$, the dependence of $N_1,N_2,K$ and $H$ on $\sigma, \epsilon, 1-\gamma$ are as follows.
	\begin{equation}\label{eq_require1}
		N_1=\mathcal{O}(\frac{\sigma^2}{\epsilon^2})\quad N_2=\mathcal{O}(\frac{M^3}{(1-\gamma)^6\epsilon^2})\quad K=\mathcal{O}(\frac{M}{(1-\gamma)^2\epsilon^2})\quad
		H=\mathcal{O}(\log\frac{M}{(1-\gamma)\epsilon})
	\end{equation}
	\subsection{Bounding the Norm of Estimated Gradient}
	\begin{equation}
	\begin{split}
	\frac{B\eta}{2K}\sum_{k=0}^{K-1}\Vert \omega^k\Vert^2&\leq\frac{B\eta}{2}\bigg[\frac{3}{K}\sum_{k=0}^{K-1}\Vert \omega^k-\tilde{\omega}^k\Vert^2+\frac{3}{K}\sum_{k=0}^{K-1}\Vert \tilde{\omega}^k-\nabla_\theta f(\bm{J}_H^{\pi_\theta})\Vert^2+\frac{3}{K}\sum_{k=0}^{K-1}\Vert \nabla_\theta f(\bm{J}_H^{\pi_\theta}))\Vert^2\bigg]\\
	&\leq \frac{B\eta}{2}\bigg[3\frac{M^2G^2L_f^2}{(1-\gamma)^4}\epsilon'^2+3\frac{\sigma^2}{N_1}+3\frac{\frac{\mathbf{E}[f(\bm{J}_H(\theta^K))-f(\bm{J}_H(\theta^0))]}{K}+(\eta+2L_J\eta^2)[\frac{M^2G^2L_f^2}{(1-\gamma)^4}\epsilon'^2+\frac{\sigma^2}{N_1}]}{\frac{\eta}{2}-L_J\eta^2}\bigg]\\
	&=B\eta\bigg[6\frac{M^2G^2L_f^2}{(1-\gamma)^4}\epsilon'^2+6\frac{\sigma^2}{N_1}+24L_J\frac{\mathbf{E}[f(\bm{J}_H(\theta^K))-f(\bm{J}_H(\theta^0))]}{K}\bigg]
	\end{split}
	\end{equation}
	Given the fixed $\epsilon$, choose the value for $\epsilon',N_1,K$ as follows,
	\begin{equation}
	\epsilon'^2\leq\frac{(1-\gamma)^4}{M^2G^2L_f^2}\cdot\frac{1}{6B\eta}\bigg(\frac{\epsilon}{9}\bigg)
	\end{equation}
	\begin{equation}
	N_1\geq\frac{54\sigma^2}{\epsilon}
	\end{equation}
	\begin{equation}
	K\geq\frac{216L_J\mathbf{E}[f(\bm{J}_H(\theta^K))-f(\bm{J}_H(\theta^0))]}{\epsilon}
	\end{equation}
	then we have
	\begin{equation}
	\frac{1}{K}\sum_{k=0}^{K-1}\mathbf{E}[\Vert\omega^k]\Vert^2\leq \frac{\epsilon}{3}
	\end{equation}
	Given the choice of $\epsilon',N_1,K$, the dependence of $N_1,N_2,K$ and $H$ on $\sigma, \epsilon, 1-\gamma$ are as follows.
	\begin{equation}\label{eq_require2}
	N_1=\mathcal{O}(\frac{\sigma^2}{\epsilon})\quad N_2=\mathcal{O}(\frac{M^3}{(1-\gamma)^6\epsilon})\quad K=\mathcal{O}(\frac{M}{(1-\gamma)^2\epsilon})\quad
	H=\mathcal{O}(\log\frac{M}{(1-\gamma)\epsilon})
	\end{equation}
	\subsection{Bounding the KL Divergence}
	It is obvious if we choose
	\begin{equation}
	K\geq\frac{3\mathbf{E}_{s\sim d_\rho^{\pi^*}}[KL(\pi^*(\cdot\vert s)\Vert\pi_{\theta^0})]}{\eta\epsilon(\cdot\vert s)}
	\end{equation}
	then
	\begin{equation}
	\frac{1}{\eta K}\mathbf{E}_{s\sim d_\rho^{\pi^*}}[KL(\pi^*(\cdot\vert s)\Vert\pi_{\theta^0})]\leq \frac{\epsilon}{3}
	\end{equation}
	In other word, the dependence of $K$ on $\epsilon$ is
	\begin{equation}\label{eq_require3}
		K=\mathcal{O}(\frac{B}{\epsilon})
	\end{equation}
\section{First Order Stationary Result for Policy Gradient}\label{sec_app_first_order}
\begin{lemma}\label{lem_fisrt_order}
	The policy gradient algorithm can achieve first-order stationary. More formally, if we choose the step size $\eta=\frac{1}{4L_J}$ and
	\begin{equation}
		N_1=\mathcal{O}(\frac{\sigma^2}{\epsilon})\quad
		N_2=\mathcal{O}(\frac{M^3}{(1-\gamma)^6\epsilon})\quad
		K=\mathcal{O}(\frac{M}{(1-\gamma)^2\epsilon})
	\end{equation}
	then,
	\begin{equation}
		\frac{1}{K}\sum_{k=0}^{K-1}\mathbf{E}[\Vert\nabla_\theta f(\bm{J}_H(\theta^k))\Vert^2]\leq\epsilon
	\end{equation}
\end{lemma}
\begin{proof}
	Recall the definition of $\omega^k$ and $\tilde{\omega}^k$ in Eq. \eqref{eq:def_omega^k} and \eqref{eq:def_tilde_omega^k}, respectively. By Lemma \ref{lem_smooth}, we have
	\begin{equation}
		\begin{split}
		f(\bm{J}_H(\theta^{k+1}))&\geq f(\bm{J}_H(\theta^{k}))+\left<\nabla_\theta f(\bm{J}_H(\theta^k)),\theta^{k+1}-\theta^k\right>-\frac{L_J}{2}\Vert \theta^{k+1}-\theta^k\Vert^2\\
		&=f(\bm{J}_H(\theta^{k}))+\eta\left<\nabla_\theta f(\bm{J}_H(\theta^k)),\omega^k\right>-\frac{L_J\eta^2}{2}\Vert \omega^k\Vert^2\\
		&\overset{(a)}=f(\bm{J}_H(\theta^{k}))+\eta\left<\nabla_\theta f(\bm{J}_H(\theta^k)),\omega^k-\nabla_\theta f(\bm{J}_H(\theta^k))+\nabla_\theta f(\bm{J}_H(\theta^k))\right>\\
		&\quad-\frac{L_J\eta^2}{2}\Vert \omega^k-\nabla_\theta f(\bm{J}_H(\theta^k))+\nabla_\theta f(\bm{J}_H(\theta^k))\Vert^2\\
		&\overset{(b)}\geq f(\bm{J}_H(\theta^{k}))+\eta\Vert\nabla_\theta f(\bm{J}_H(\theta^k))\Vert^2-\eta\vert\left<\nabla_\theta f(\bm{J}_H(\theta^k)),\omega^k-\nabla_\theta f(\bm{J}_H(\theta^k))\right>\vert\\
		&\quad-L_J\eta^2\bigg(\Vert \omega^k-\nabla_\theta f(\bm{J}_H(\theta^k))\Vert^2+\Vert\nabla_\theta f(\bm{J}_H(\theta^k))\Vert^2\bigg)\\
		&\geq f(\bm{J}_H(\theta^{k}))+\eta\Vert\nabla_\theta f(\bm{J}_H(\theta^k))\Vert^2-\frac{\eta}{2}\Vert\nabla_\theta f(\bm{J}_H(\theta^k))\Vert^2-\frac{\eta}{2}\Vert \omega^k-\nabla_\theta f(\bm{J}_H(\theta^k))\Vert^2\\
		&\quad-L_J\eta^2\bigg(\Vert \omega^k-\nabla_\theta f(\bm{J}_H(\theta^k))\Vert^2+\Vert\nabla_\theta f(\bm{J}_H(\theta^k))\Vert^2\bigg)\\
		&= f(\bm{J}_H(\theta^{k}))+(\frac{\eta}{2}-L_J\eta^2)\Vert\nabla_\theta f(\bm{J}_H(\theta^k))\Vert^2-(\frac{\eta}{2}+L_J\eta^2)\Vert \omega^k-\nabla_\theta f(\bm{J}_H(\theta^k))\Vert^2\\
		&\overset{(c)}\geq f(\bm{J}_H(\theta^{k}))+(\frac{\eta}{2}-L_J\eta^2)\Vert\nabla_\theta f(\bm{J}_H(\theta^k))\Vert^2-(\eta+2L_J\eta^2)\Vert \omega^k-\tilde{\omega}^k\Vert^2\\
		&\quad -(\eta+2L_J\eta^2)\Vert \tilde{\omega}^k-\nabla_\theta f(\bm{J}_H(\theta^k))\Vert^2\\
		&\overset{(d)}\geq f(\bm{J}_H(\theta^{k}))+(\frac{\eta}{2}-L_J\eta^2)\Vert\nabla_\theta f(\bm{J}_H(\theta^k))\Vert^2-(\eta+2L_J\eta^2)\frac{M^2G^2L_f^2}{(1-\gamma)^4}\epsilon'^2\\
		&-(\eta+2L_J\eta^2)\Vert \tilde{\omega}^k-\nabla_\theta f(\bm{J}_H(\theta^k))\Vert^2
		\end{split}
	\end{equation}
	where the step (a) holds by $\theta^{k+1}=\theta^k+\eta \omega^k$. Step (b) and (c) holds by Cauchy-Schwarz Inequality. Step (d) holds by Lemma \ref{lem_bound_bias1}. Then, take expectation with respect to the trajectories $\tau_i,\tau_j$ (Recall that $\theta^k,\theta^{k+1}$ is a function of $\tau_i,\tau_j$), we have
	\begin{equation}\label{eq:first_order_telescope}
		\begin{split}
		\mathbf{E}[f(\bm{J}_H(\theta^{k+1}))]&\geq \mathbf{E}[f(\bm{J}_H(\theta^{k}))]+(\frac{\eta}{2}-L_J\eta^2)\mathbf{E}[\Vert\nabla_\theta f(\bm{J}_H(\theta^k))\Vert^2]-(\eta+2L_J\eta^2)\frac{M^2G^2L_f^2}{(1-\gamma)^4}\epsilon'^2\\
		&\quad-(\eta+2L_J\eta^2)\mathbf{E}[\Vert \tilde{g}^k-\nabla_\theta f(\bm{J}_H(\theta^k))\Vert^2]\\
		&\geq \mathbf{E}[f(\bm{J}_H(\theta^{k}))]+(\frac{\eta}{2}-L_J\eta^2)\mathbf{E}[\Vert\nabla_\theta f(\bm{J}_H(\theta^k))\Vert^2]-(\eta+2L_J\eta^2)\frac{M^2G^2L_f^2}{(1-\gamma)^4}\epsilon'^2\\
		&\quad-(\eta+2L_J\eta^2)\frac{\sigma^2}{N_1}
		\end{split}
	\end{equation}
	where the last step holds by Assumption \ref{ass_bounded_var}. Notice that in Eq. \eqref{eq:first_order_telescope}, $\mathbf{E}[f(\bm{J}_H(\theta^{k+1}))]$ and $\mathbf{E}[f(\bm{J}_H(\theta^{k}))]$ give a recursive form. Thus, telescoping from $k=0$ to $k=K-1$, we have
	\begin{equation}\label{eq:first_order}
		\frac{\mathbf{E}[f(\bm{J}_H(\theta^{K}))-f(\bm{J}_H(\theta^0))]}{K}\geq (\frac{\eta}{2}-L_J\eta^2)\frac{1}{K}\sum_{k=0}^{K-1}\mathbf{E}[\Vert\nabla_\theta f(\bm{J}_H(\theta^k))\Vert^2]-(\eta+2L_J\eta^2)[\frac{M^2G^2L_f^2}{(1-\gamma)^4}\epsilon'^2+\frac{\sigma^2}{N_1}]
	\end{equation}
	and thus
	\begin{equation}
		\frac{1}{K}\sum_{k=0}^{K-1}\mathbf{E}[\Vert\nabla_\theta f(\bm{J}_H(\theta^k))\Vert^2]\leq \frac{\frac{\mathbf{E}[f(\bm{J}_H(\theta^K))-f(\bm{J}_H(\theta^0))]}{K}+(\eta+2L_J\eta^2)[\frac{M^2G^2L_f^2}{(1-\gamma)^4}\epsilon'^2+\frac{\sigma^2}{N_1}]}{\frac{\eta}{2}-L_J\eta^2}
	\end{equation}
	Taking $\eta=\frac{1}{4L_J}$ and letting $N_1=\frac{18\sigma^2}{\epsilon}$, $K=\frac{48L_J\mathbf{E}[\Vert\nabla_\theta f(\bm{J}_H(\theta^K))-\nabla_\theta f(\bm{J}_H(\theta^0))\Vert^2]}{\epsilon}$ and $\epsilon'=\frac{(1-\gamma)^2}{MGL_f}\sqrt{\frac{\epsilon}{6}}$, we have
	\begin{equation}
		\frac{1}{K}\sum_{k=0}^{K-1}\mathbf{E}[\Vert\nabla_\theta f(\bm{J}_H(\theta^k))\Vert^2]\leq\epsilon
	\end{equation}
	Recalling the definition of $N_2$ in the statement of Lemma \ref{lem_bound_bias1}, we have
	\begin{equation}
		N_2=\frac{6M^3G^2L_f^2(1-\gamma^H)^2}{(1-\gamma)^6\epsilon}\log(\frac{2MH}{p})
	\end{equation}
	Also, by the definition of $L_J$ in the lemma \ref{lem_smooth}
	\begin{equation}
		K=\frac{48MCB}{(1-\gamma)^2\epsilon}\mathbf{E}[\Vert\nabla_\theta f(\bm{J}_H(\theta^k))\Vert^2]
	\end{equation}
\end{proof}
\section{Further Discussion on all Asumptions}\label{app_discussion}
\begin{itemize}
	\item Assumption 1 is related to the bound for reward and it can always be satisfied by scalarization or shifting.
	\item Assumptions 3 and 4 are about the function class. They require a concave function with local-Lipschitz partial derivatives. As we discussed in the limitation, many function with regularization such as $\log(x), -x^2, \sqrt{x}, \sin(x)$ will satisfy these conditions.
	\item  The remaining 4 assumptions limit the policy parameterization 
	\begin{itemize}
		\item We would like to say assumption 5 can be implied by assumption 2. This is because $\tilde{g}(\tau_i^H,\tau_j^H\vert \theta)$ is bounded under assumption 2 and thus the variance is also bounded. 
		\item The property that the likelihood is smooth and the gradient of it is bounded can be satisfied by Gaussian policy (Appendix C in \cite{Tianbing2017}) and log-linear policy class (Remark 6.7 in \cite{Alekh2020}).
		\item The positive definite property of Fisher matrix can also be satisfied by Gaussian Policy (Appendix B.2 in \cite{Yanli2020}) and log-linear policy class (Assumption 6.5 part 3 in \cite{Alekh2020}).
		\item For the last assumption, the intuition for $\epsilon_{bias}=0$ is that the difference between Eq. 31 and 32 is only the distribution of state and action. What we want here is that Eq. 31 is equal to 0 for any distribution. Using any policy parameterizations with $\theta\in\mathbb{R}^{d}$, we have $|S|\times|A|$ equations (one corresponding to each state-action pair) with $d$ variables. If $d=|S|\times|A|$, we will have $\epsilon_{bias}=0$. Thus, any complete parameterization for tabular case will have $\epsilon_{bias}=0$. For the general case, a linear MDP \cite{jin2020provably} will also give $\epsilon_{bias}=0$ as long as we use the features of the linear MDP (Remark 6.4 in \cite{Alekh2020}) and both Gaussian policy and log-linear policy can be used.
		\item Above all, Gaussian policy and log-linear policy satisfy the above 4 assumptions. 
	\end{itemize}
\end{itemize}

		\bibliographystyle{theapa}
	
		\bibliography{ref}
	
\end{document}